%% file: paper.tex
\documentclass[letterpaper]{article} 
\usepackage{aaai24}  
\usepackage{times}  
\usepackage{helvet}  
\usepackage{courier}  
\usepackage[hyphens]{url}  
\usepackage{graphicx} 
\usepackage{natbib}  
\usepackage{caption} 
\usepackage{algorithm}
\usepackage[noEnd,indLines,commentColor=tblue2]{algpseudocodex}
\usepackage{eqparbox}

\usepackage{newfloat}
\usepackage{listings}
\DeclareCaptionStyle{ruled}{labelfont=normalfont,labelsep=colon,strut=off} 
\floatstyle{ruled}
\newfloat{listing}{tb}{lst}{}
\floatname{listing}{Listing}
\newcommand{\mailtodomain}[1]{\href{mailto:#1@monash.edu}{\texttt{#1}}}

\title{Anytime Approximate Formal Feature Attribution}
\author {
  Jinqiang Yu,
  Graham Farr,
  Alexey Ignatiev,
  Peter J. Stuckey
}
\affiliations {
  Department of Data Science and AI, Faculty of IT \\
  Monash University, Melbourne, Victoria, Australia \\
  {\small\texttt{\{}\mailtodomain{jinqiang.yu}\texttt{,}\mailtodomain{graham.farr}\texttt{,}\mailtodomain{alexey.ignatiev}\texttt{,}\mailtodomain{peter.stuckey}\texttt{\}}\texttt{@monash.edu}}
}
\usepackage{bibentry}

\usepackage{hyperref}

\usepackage{adjustbox}
\usepackage{appendix}
\usepackage{amsmath}
\usepackage{amssymb}
\usepackage{amsfonts}
\usepackage{amsthm}
\usepackage{cleveref}
\usepackage{thmtools}
\usepackage{subcaption}
\usepackage{tikz}
\usetikzlibrary{arrows,decorations,fit,positioning,shapes}
\usepackage{empheq,fancybox}
\usepackage{adjustbox}
\usepackage{appendix}
\usepackage{stmaryrd}
\usepackage{multirow}
\usepackage{bm}
\usepackage{xcolor}         
\usepackage{booktabs}       
\usepackage{url}            
\usepackage[english]{babel}
\usepackage{xspace}

\input{defs}

\input{tango}

\addto\extrasenglish{%

}

\graphicspath{{plots/},{plots/28,28_mnist_1v3_flip},{plots/28,28_mnist_1v7_flip},{plots/10,10_mnist_1v3_flip},{plots/10,10_mnist_1v7_flip}}

\DeclareGraphicsExtensions{.pdf}

\begin{document}

\maketitle

\input{abs}
\input{intro}
\input{prelim}
\input{approx}
\input{res}
\input{conc}
\bibliography{refs}

\input{appendix}


\end{document}

%% file: defs.tex
\newtheorem{theorem}{Theorem}

\newtheorem{corollary}{Corollary}

\newtheorem*{summary*}{Summary}

\declaretheoremstyle[
headpunct={},
headfont=\bfseries,
notefont=\bfseries,
bodyfont=\normalfont,
headformat=\NAME~\NUMBER:\NOTE.
]{def}

\newcommand{\fml}[1]{{\mathcal{#1}}}

\newcommand{\mbf}[1]{\ensuremath\mathbf{#1}}
\newcommand{\mbb}[1]{\ensuremath\mathbb{#1}}

\newcommand{\ignore}[1]{}

\newcommand{\hs}{\textrm{HS}}
\newcommand{\mins}{\textrm{mins}}
\newcommand{\mhs}{\textrm{MHS}}
\newcommand{\ffa}{\textrm{ffa}}

\newcommand{\exps}{\ensuremath\mbb{E}}
\newcommand{\axps}{\ensuremath\mbb{A}}
\newcommand{\cxps}{\ensuremath\mbb{C}}
\newcommand{\maxp}{\textrm{MARCO-A}\xspace}
\newcommand{\mcxp}{\textrm{MARCO-C}\xspace}
\newcommand{\mswitch}{\textrm{MARCO-S}\xspace}

\DeclareMathOperator*{\limply}{\rightarrow}

\def\squareforqed{\hbox{\rlap{$\sqcap$}$\sqcup$}}
\def\qed{\ifmmode\squareforqed\else{\unskip\nobreak\hfil
\penalty50\hskip1em\null\nobreak\hfil\squareforqed
\parfillskip=0pt\finalhyphendemerits=0\endgraf}\fi}

%% file: tango.tex
\definecolor{tyellow1}{HTML}{FCE94F}
\definecolor{tyellow2}{HTML}{EDD400}
\definecolor{tyellow3}{HTML}{C4A000}
\definecolor{torange1}{HTML}{FCAF3E}
\definecolor{torange2}{HTML}{F57900}
\definecolor{torange3}{HTML}{C35C00}
\definecolor{tbrown1}{HTML}{E9B96E}
\definecolor{tbrown2}{HTML}{C17D11}
\definecolor{tbrown3}{HTML}{8F5902}
\definecolor{tgreen1}{HTML}{8AE234}
\definecolor{tgreen2}{HTML}{73D216}
\definecolor{tgreen3}{HTML}{4E9A06}
\definecolor{tblue1}{HTML}{729FCF}
\definecolor{tblue2}{HTML}{3465A4}
\definecolor{tblue3}{HTML}{204A87}
\definecolor{tpurple1}{HTML}{AD7FA8}
\definecolor{tpurple2}{HTML}{75507B}
\definecolor{tpurple3}{HTML}{5C3566}
\definecolor{tred1}{HTML}{EF2929}
\definecolor{tred2}{HTML}{CC0000}
\definecolor{tred3}{HTML}{A40000}
\definecolor{tlgray1}{HTML}{EEEEEC}
\definecolor{tlgray2}{HTML}{D3D7CF}
\definecolor{tlgray3}{HTML}{BABDB6}
\definecolor{tdgray1}{HTML}{888A85}
\definecolor{tdgray2}{HTML}{555753}
\definecolor{tdgray3}{HTML}{2E3436}

%% file: abs.tex
\begin{abstract}\label{sec:abs}
    Widespread use of artificial intelligence (AI) algorithms and machine learning (ML) models on the one hand and a number of crucial issues pertaining to them warrant the need for explainable artificial intelligence (XAI).
    A key explainability question is: given this decision was made, what are the input features which contributed to the decision?
    Although a range of XAI approaches exist to tackle this problem, most of them have significant limitations.
    Heuristic XAI approaches suffer from the lack of quality guarantees, and often try to approximate Shapley values, which is not the same as explaining which features contribute to a decision. 
    A recent alternative is so-called formal feature attribution (FFA), which defines feature importance as the fraction of formal abductive explanations (AXp's) containing the given feature. This measures feature importance from the view of formally reasoning about the model's behavior.
    %
    It is challenging to compute FFA using its definition because that involves counting AXp's, although one can approximate it.
    Based on these results, this paper makes several contributions.
    First, it gives compelling evidence that computing FFA is intractable, even if the set of contrastive formal explanations (CXp's) is provided, by proving that the problem is \#P-hard.
    Second, by using the duality between AXp's and CXp's, it proposes an efficient heuristic to switch from CXp enumeration to AXp enumeration on-the-fly resulting in an adaptive explanation enumeration algorithm effectively approximating FFA in an anytime fashion.
    Finally, experimental results obtained on a range of widely used datasets demonstrate the effectiveness of the proposed FFA approximation approach in terms of the error of FFA approximation as well as the number of explanations computed and their diversity given a fixed time limit.
\end{abstract}

%% file: intro.tex
\section{Introduction} \label{sec:intro}


The rise of the use of artificial intelligence (AI) and machine learning (ML) methods to help interpret data and make decisions has exposed a keen need for these algorithms to be able to explain their decisions/judgements.
Lack of explanation of opaque and complex models leads to lack of trust, and allows the models to encapsulate unfairness, discrimination and other unwanted properties learnt from the data or through training.

For a classification problem a key explainability question is: ``given a decision was made (a class was imputed to some data instance), what are the features that contributed to the decision?''.
A more complex question is: ``given the decision was made, how important was each feature in making that decision?''.
There are many heuristic approaches to answering this question, mostly based on sampling around the instance~\cite{guestrin-kdd16}, and attempting to approximate Shapley values~\cite{lundberg-nips17}.
But there is strong evidence that Shapley values do not really compute the importance of a feature to a decision~\cite{huang-corr23,huang-corr23b}.

Formal approaches to explainability are able compute formal \emph{abductive explanations} (AXp's) for a decision, that is a minimal set of features which are enough to ensure the same decision will be made~\cite{darwiche-ijcai18,inms-aaai19,msi-aaai22}.
They can also compute formal \emph{contrastive explanations} (CXp's), that is a minimal set of features, one of which must change in order to change the decision~\cite{miller-aij19,inams-aiia20}.
A wealth of algorithms originating from the area of dealing with over-constrained systems~\cite{bs-dapl05,liffiton-jar08,blms-aicom12,mshjpb-ijcai13,lpmms-cj16} can be applied for the computation and enumeration of AXp's and CXp's~\cite{msi-aaai22}.
Here, enumeration of formal explanations builds on the use of the minimal hitting set duality between AXp's and CXp's~\cite{reiter-aij87,lpmms-cj16,inams-aiia20}.
Until recently there was no formal approach to ascribing importance to features.

A recent and attractive approach to formal feature attribution, called FFA~\cite{yis-corr23}, is simple. Compute all the abductive explanations for a decision, then the importance of a feature for the decision is simply the proportion of abductive explanations in which it appears.
FFA is crisply defined, and easy to understand, but it is challenging to compute, as deciding if a feature has a non-zero attribution
is at least as hard as deciding feature relevancy~\cite{huang-tacas23,yis-corr23}.


Yu \emph{et al.}~\cite{yis-corr23} show that FFA can be efficiently computed by making use of the hitting set duality between AXp's and CXp's.
By trying to enumerate CXp's, a side effect of the algorithm is to discover AXp's. In fact, the algorithm will usually find many AXp's before finding the first CXp. The AXp's are guaranteed to be diverse, since they need to be broad in scope to ensure that the CXp is large enough to hit all AXp's that apply to the decision.

Using AXp's collected as a side effect of CXp enumeration is effective at the start of the enumeration. But as we find more and more AXp's as side effects we eventually get to a point where many more CXp's are generated than AXp's.  Experimentation shows that if we wish to enumerate all AXp's then indeed we should not rely on the side effect behavior, but simply enumerate AXp's directly.
This leads to a quandary: to get fast accurate approximations of FFA we wish to enumerate CXp's and generate AXp's as a side effect.  But to compute the final correct FFA we wish to compute all AXp's, and we are better off directly enumerating AXp's.

In this paper, we develop an \emph{anytime} approach to computing approximate FFA, by starting with CXp enumeration, and then dynamically switching to AXp enumeration when the rate of AXp discovery by CXp enumeration drops. In doing so, we are able to quickly get accurate approximations, but also arrive to the full set of AXp's quicker than pure CXp enumeration.
As direct CXp enumeration is feasible to do without the need to resort to the hitting set duality~\cite{inams-aiia20,msi-aaai22}, one may want to estimate FFA by first enumerating CXp's. The second contribution of this paper is to investigate this alternative approach and to show that even if a(n) (in)complete set of CXp's is given, determining FFA is computationally expensive being \#P-hard even if all CXp's are of size two.

%% file: prelim.tex
\section{Preliminaries} \label{sec:prelim}


Here we introduce the notation and background on formal XAI in order to define formal feature attribution (FFA).

\subsection{Classification Problems}

We assume classification problems classify data instances into classes $\fml{K}$ where $|\fml{K}| = k \geq 2$.
We are given a set of $m$ features $\fml{F}$, where the value of feature $i \in \fml{F}$ comes from a domain $\mbb{D}_i$, which may be Boolean, (bounded) integer or (bounded) real.
The \emph{complete feature space} is defined by $\mathbb{F}\triangleq\prod_{i=1}^{m}\mbb{D}_i$.

A \emph{data point} in feature space is denoted $\mbf{v} = (v_1, \ldots, v_m)$ where $v_i \in \mbb{D}_i, 1 \leq i \leq m$.
An \emph{instance} of the classification problem is a pair of feature vector and its corresponding class, i.e.
$(\mbf{v}, c)$, where $\mbf{v}\in\mbb{F}$ and $c\in \fml{K}$.

We use the notation $\mbf{x} = (x_1, \ldots, x_m)$ to represent an arbitrary point in feature space, where each $x_i$ will take a value from $\mbb{D}_i$.

A \emph{classifier} is a total function from feature space to class: $\kappa: \mathbb{F} \rightarrow \fml{K}$.
Many approaches exist to define classifiers including decision sets~\cite{clark-ewsl91,lakkaraju-kdd16}, decision lists~\cite{rivest-ml87}, decision trees~\cite{rivest-ipl76}, random forests~\cite{friedman-tas01}, boosted trees~\cite{guestrin-kdd16a}, and neural nets~\cite{hinton-icml10,hcseyb-neurips16}.

\subsection{Formal Explainability}

Given a data point $\mbf{v}$, classifier $\kappa$ classifies it as class $\kappa(\mbf{v})$.
A \emph{post hoc explanation} of the behavior of $\kappa$
on data point $\mbf{v}$ tries to explain the behavior of $\kappa$ on this instance.  We consider two forms of formal explanation answering \emph{why} and \emph{why not} (or \emph{how}) questions.

An \emph{abductive explanation} (AXp) is a minimal set of features $\fml{X}$ such that any data point sharing the same feature values with $\mbf{v}$ on these features is guaranteed to be assigned the same class by $c = \kappa(\mbf{v})$~\cite{darwiche-ijcai18,inms-aaai19}.
Formally, $\fml{X}$ is a subset-minimal set of features such that:

\begin{equation} \label{eq:axp}
  \forall(\mbf{x} \in \mbb{F}). \left[\bigwedge\nolimits_{i \in \fml{X}}
  (x_i = v_i)\right] \limply (\kappa(\mbf{x}) = c)
\end{equation}

A dual concept of \emph{contrastive explanations} (CXp's) helps us understand \emph{how} to reach another prediction~\cite{miller-aij19,inams-aiia20,msi-aaai22}.
A \emph{contrastive explanation} (CXp) for the classification of data point $\mbf{v}$ as class $c = \kappa(\mbf{v})$ is a minimal set of features such that at least one must change in order that $\kappa$ will return a different class.
Formally, a CXp is a subset minimal set of features $\fml{Y}$
such that
\begin{equation} \label{eq:cxp}
  \exists(\mbf{x}\in\mbb{F}).\left[\bigwedge\nolimits_{i\not\in\fml{Y}}(x_i=v_i)\right]\land(\kappa(\mbf{x})\not=c)
\end{equation}

Interestingly, the set $\axps$ of all AXp's $\fml{X}$ explaining classification $\kappa(\mbf{v}) = c$ and the set $\cxps$ of
all CXp's $\fml{Y}$ explaining the same classification enjoy a \emph{minimal hitting set duality}~\cite{inams-aiia20}.
That is $\axps = \mhs(\cxps)$ and is $\cxps = \mhs(\axps)$
where $\mhs(S)$ returns the minimal hitting sets of $S$, that is the minimal sets that share an element with each subset in $S$.
More formally, $\hs(S) = \{ t \subseteq (\cup S) ~|~ \forall s \in S,~ t \cap s \neq \emptyset\})$ and $\mins(S) = \{ s \in S ~|~ \forall t \subsetneq s,~t \not\in S\}$ returns the subset minimal elements of a set of sets, and $\mhs(S) = \mins(\hs(S))$.
This property can be made use of in computing or enumerating AXp's and/or CXp's.

A growing body of recent work on formal explanations is represented
(but not limited)
by~\cite{msgcin-nips20,msgcin-icml21,ims-ijcai21,ims-sat21,barcelo-nips21,kutyniok-jair21,darwiche-jair21,kwiatkowska-ijcai21,mazure-cikm21,tan-nips21,iims-jair22,rubin-aaai22,iisms-aaai22,hiicams-aaai22,msi-aaai22,an-ijcai22,leite-kr22,jpms-rw22,barcelo-nips22,jpms-corr22,yisnms-aaai23,jpms-aaai23,jpms-aij23,msi-fai23,ihincms-ajar23,izza-corr23,jpms-corr23,huang-corr23b,yis-cp23,jpms-ecai23,jpms-aij23,izza-corr23b}.

\subsection{Formal Feature Attribution}

Given the definition of AXp's above, we can now illustrate the \emph{formal feature attribution} (FFA) function by Yu \emph{et al}~\cite{yis-corr23}.\footnote{Measuring feature importance from the perspective of formal explainability was independently studied in~\cite{izza-corr23b}.}
Denoted as $\ffa_\kappa(i, (\mbf{v}, c))$, it returns for a classification $\kappa(\mbf{v}) = c$ how important feature $i \in \fml{F}$ is in making this classification, defined as the proportion of AXp's for the classification $\axps_\kappa(\mbf{v}, c)$, which include feature $i$, i.e.
\begin{equation} \label{eq:ffa}
    \ffa_\kappa(i,(\mbf{v}, c)) = \frac{|\{ \fml{X} ~|~ \fml{X} \in
      \axps_\kappa(\mbf{v}, c), i \in \fml{X}) |}{|
    \axps_\kappa(\mbf{v}, c)|}
\end{equation}

\subsection{Computing FFA} \label{sec:ffa}

Yu \emph{et al}~\cite{yis-corr23} define an anytime algorithm for computing FFA
shown in \Cref{alg:enum}.
The algorithm collects AXp's $\mathbb{A}$ and CXp's $\mathbb{C}$.
They are initialized to empty.
While we still have resources, we generate a minimal hitting set $\fml{Y} \in \mhs(\mathbb{A})$ of
the current known AXp's $\mathbb{A}$ and not already in $\mathbb{C}$ with the call
$\textsc{MinimalHS}(\mathbb{A},\mathbb{C})$.
If no (new) hitting set exists then we are finished and exit the loop.
We then check if \eqref{eq:cxp} holds in which case we add the candidate to the set of CXp's $\mathbb{C}$.
Otherwise, we know that $\fml{F} \setminus \fml{Y}$ is a correct (non-minimal) abductive explanation, i.e. it satisfies \eqref{eq:axp}.
We use the call $\textsc{ExtractAXp}$ to minimize the resulting explanation, returning an AXp $\fml{X}$ which is added to the collection of AXp's $\mathbb{A}$.
$\textsc{ExtractAXp}$
tries to remove features $i$ from $\fml{F} \setminus \fml{Y}$ one by one while still satisfying \eqref{eq:axp}.
When resources are exhausted, the loop exits and we return the set of AXp's and CXp's currently discovered.

\input{algs/enum}

\subsection{Graph-Related Notation}

The paper uses some (undirected) graph-theoretic concepts.
A graph is defined as a tuple, $G=(V,E)$, where $V$ is a finite set of vertices and $E$ is a finite set
of unordered pairs of vertices.
%
For simplicity, $uv$ denotes an edge $\{u,v\}$ of $E$.
Given a graph $G=(V,E)$, a \emph{vertex cover} $X\subseteq V$ is such that for each $uv\in E$, $\{u,v\}\cap X\not=\emptyset$.
A \emph{minimal} vertex cover is a vertex cover that is minimal wrt. set inclusion.
%

\subsection{The Complexity of Counting}

The class \#P consists of functions that count accepting computations
of polynomial-time nondeterministic Turing machines \cite{valiant-tcs79}.
A problem is \textit{\#P-hard} if every problem in \#P is polynomial-time Turing reducible to it; if it also belongs
to \#P then it is \textit{\#P-complete}.


\#P-hardness is usually regarded as stronger evidence of intractability than NP-hardness or indeed hardness for any level of the Polynomial Hierarchy.

%% file: algs/enum.tex
\newcommand{\Break}{\textbf{break}}
\algnewcommand{\IfThen}[2]{
  \State \algorithmicif\ #1\ \algorithmicthen\ #2}
\algnewcommand{\IfThenElse}[3]{
  \State \algorithmicif\ #1\ \algorithmicthen\ #2\ \algorithmicelse\ #3}

\begin{algorithm}[t]
  \begin{algorithmic}[1]
    \Procedure{XpEnum}{$\kappa$, $\mbf{v}$, $c$}
      \State $\left(\axps, \cxps\right)\gets (\emptyset, \emptyset)$ \label{ln:init}
      \While{resources available}
        \State $\fml{Y}\gets\Call{MinimalHS}{\axps, \cxps}$ \label{ln:mhs}
        \IfThen{$\fml{Y}=\bot$}{\Break} \label{ln:nocand}
        \If{$\exists(\mbf{x}\in\mbb{F}).\bigwedge\nolimits_{i\not\in\fml{Y}}(x_i=v_i)\land(\kappa(\mbf{x})\not=c)$} \label{ln:check}
          \State $\cxps\gets\cxps\cup\{\fml{Y}\}$ \label{ln:cxprec}
        \Else
          \State $\fml{X}\gets\Call{ExtractAXp}{\fml{F}\setminus\fml{Y},\kappa,\mbf{v},c}$ \label{ln:extract}
          \State $\axps\gets\axps\cup\{\fml{X}\}$ \label{ln:axprec}
        \EndIf
      \EndWhile
      \Return{$\axps$, $\cxps$}
    \EndProcedure
  \end{algorithmic}
  \caption{Anytime Explanation Enumeration}
  \label{alg:enum}
\end{algorithm}

%% file: approx.tex
\section{Approximate Formal Feature Attribution} \label{sec:approx}

Facing the need to compute (exact or approximate) FFA values, one may think of a possibility to first enumerate CXp's and then apply the minimal hitting set duality between AXp's and CXp's to determine FFA, without explicitly computing $\axps = \mhs(\cxps)$.
This looks plausible given that CXp enumeration can be done directly, without the need to enumerate AXp's~\cite{inams-aiia20}.
However, as \Cref{sec:dual} argues, computing FFA given a set of CXp's turns out to be computationally difficult, 
being (roughly) at least as hard as counting the minimal hitting sets $\mhs(\cxps)$.
Hence, \Cref{sec:switch} approaches the problem from a different angle by efficient exploitation of the eMUS- or MARCO-like setup~\cite{pms-aaai13,liffiton-cpaior13,lpmms-cj16,inams-aiia20} and making the algorithm \emph{switch} from CXp enumeration to AXp enumeration on the fly. 

\input{dual}

\input{switch}

%% file: dual.tex
\subsection{Duality-Based Approximation is Hard} \label{sec:dual}


We show that determining $\hbox{ffa}_{\kappa}(i,(\mathbf{v},c))$ from $\mathbb{C}$ is \#P-hard even when all
CXp's have size two.  In that special case, the CXp's may be treated as the edges of a graph,
which we denote by $G(\mathcal{F},\kappa,\mathbf{v},c)$, with vertex set $\mathcal{F}$.
The minimal hitting
set duality between the CXp's and AXp's then implies that the AXp's $\fml{X} \in \mhs(\cxps)$ are precisely the minimal
vertex covers of $G(\mathcal{F},\kappa,\mathbf{v},c)$.
It is known that determining the
number of minimal vertex covers in a graph is \#P-complete (even for bipartite graphs);
this is implicit in~\cite{provan-sicomp83}, as noted for example in~\cite[p.~400]{vadhan-sicomp01}.

When all CXp's have size 2, the formal feature attribution $\hbox{ffa}_{\kappa}(i,(\mathbf{v},c))$
is just the proportion of minimal vertex covers of $G(\mathcal{F},\kappa,\mathbf{v},c)$
that contain the vertex $i$, i.e. the vertex of $G(\mathcal{F},\kappa,\mathbf{v},c)$ that
represents the feature $i\in\mathcal{F}$. To help express this in graph-theoretic language,
write $\hbox{\#mvc}(G)$ for the number of minimal vertex covers of $G$.  Write 
$\hbox{\#mvc}(G, v)$ and $\hbox{\#mvc}(G, \neg v)$ for the numbers of minimal vertex covers
of $G$ that \emph{do} and \emph{do not} contain vertex $v\in V(G)$, respectively.  Define
\begin{equation}
\label{eq:ffaG}
\hbox{ffa}(G, v) := \frac{\hbox{\#mvc}(G, v)}{\hbox{\#mvc}(G)} .
\end{equation}
Then
\[
\hbox{ffa}_{\kappa}(i,(\mathbf{v},c)) = \hbox{ffa}(G(\mathcal{F},\kappa,\mathbf{v},c), i) .
\]

Observe that
$\hbox{\#mvc}(G) = \hbox{\#mvc}(G, v)+\hbox{\#mvc}(G, \neg v)$.
Then we may rewrite \eqref{eq:ffaG} as
\begin{equation}
\label{eq:ffaG2}
\hbox{ffa}(G, v) = \frac{\hbox{\#mvc}(G, v)}{\hbox{\#mvc}(G, v)+\hbox{\#mvc}(G, \neg v)}.
\end{equation}

\begin{theorem}
    Determining $\hbox{\rm ffa}(G, v)$ is \#P-hard.
\end{theorem}

\begin{proof}
We give a polynomial-time Turing reduction from the \#P-complete problem of counting minimal
vertex covers to the problem of determining ffa for a node in a graph.

Suppose we have an oracle that, when given a graph and a vertex, returns the ffa of the vertex in one time-step.

Let $G$ be a graph for which we want to count the minimal vertex covers.
Let $v$ be a non-isolated vertex of $G$.
(If none exists, the problem is trivial.)  Put
\begin{eqnarray*}
x  & = & \hbox{\#mvc}(G, \neg v),   \\
y  & = & \hbox{\#mvc}(G,v),
\end{eqnarray*}
so that $\hbox{\#mvc}(G)=x+y$.  It is routine to show
that $x,y>0$.  Initially, $x$ and $y$ are unknown.
Our reduction will use an ffa-oracle to gain enough information
to determine $x$ and $y$.  We will then obtain $\hbox{\#mvc}(G)=x+y$.

First, query the ffa-oracle with $G$ and vertex $v$.  It returns
\[
p := \frac{y}{x+y},
\]
by \eqref{eq:ffaG2}.  We can then recover the ratio $x/y = p^{-1}-1$.

Then we construct a new graph $G_v^{[2]}$ from $G$ as follows.
Take two disjoint copies $G_1$ and $G_2$ of $G$.  Let $v_1$ be the copy of vertex $v$ in $G_1$.
For every $w\in V(G_2)$, add an edge $v_1w$ between $v_1$ and $w$.  We query the ffa-oracle with $G_v^{[2]}$ and vertex $v_1$.  Let $q=\hbox{ffa}(G_v^{[2]}, v_1)$ be the value it returns.

If a minimal vertex cover $X$ of $G_v^{[2]}$ contains $v_1$ then
all the edges from $v_1$ to $G_2$ are covered.
The restriction of $X$ to $G_1$ must be a minimal vertex cover of $G_1$ that contains $v_1$,
and the number of these is just $\hbox{\#mvc}(G,v)=y$.  The restriction of $X$ to $G_2$ must
just be a vertex cover of $G_2$, without any further restriction, and the number of these is
just $\hbox{\#mvc}(G)=x+y$.  These two restrictions of $X$ can be chosen independently to give
all possibilities for $X$.  So
\[
\hbox{\#mvc}(G_v^{[2]},v_1)=y(x+y).
\]

If a minimal vertex cover $X$ of $G_v^{[2]}$ does not contain $v_1$ then the edges
$v_1w$, $w\in V(G_2)$, are not covered by $v_1$.  So each $w\in V(G_2)$ must be in $X$,
which serves to cover not only those edges but also all edges in $G_2$.  The restriction of $X$
to $G_1$ must just be a vertex cover of $G_1$ that does not contain $v_1$, and there
are $\hbox{\#mvc}(G,\neg v)=x$ of these.  Again, the two
restrictions of $X$ are independent.  So
\[
\hbox{\#mvc}(G_v^{[2]}, \neg v_1)=x.
\]

Therefore
\[
q = \frac{y(x+y)}{x + y(x+y)},
\]
by \eqref{eq:ffaG2} (applied this time to $G_v^{[2]}$),
so
\[
x+y = \frac{x/y}{q^{-1}-1} = \frac{p^{-1}-1}{q^{-1}-1}.
\]
We can therefore compute $x+y$ from the values $p$ and $q$ returned by our two oracle calls.
We therefore obtain $\hbox{\#mvc}(G)$.  The entire computation is polynomial time.
\end{proof}

\begin{corollary}
    Determining $\hbox{\rm ffa}_{\kappa}(i,(\mathbf{v},c))$ from the set of CXp's is \#P-hard, even
    if all CXp's have size 2.
\end{corollary}

%% file: switch.tex
\subsection{Switching from CXp to AXp Enumeration} \label{sec:switch}

As discussed in \Cref{sec:prelim},~\cite{yis-corr23} proposed to apply implicit hitting set enumeration for approximating FFA thanks to the duality between AXp's and CXp's.
The approach builds on the use of the MARCO algorithm~\cite{pms-aaai13,liffiton-cpaior13,lpmms-cj16} in the anytime fashion, i.e. collects the sets of AXp's and CXp's and stops upon reaching a given resource limit.
As MARCO can be set to target enumerating either AXp's or CXp's depending on user's preferences, the dual explanations will be collected by the algorithm as a \emph{side effect}.
Quite surprisingly, the findings of~\cite{yis-corr23} show that for the purposes of \emph{practical} FFA approximation it is beneficial to target CXp enumeration and get AXp's by duality.
An explanation offered for this by~\cite{yis-corr23} is that MARCO has to collect a large number of dual explanations before the minimal hitting sets it gets may realistically be the target explanations.

Our practical observations confirm the above.
Also note that the AXp's enumerated by MARCO need to be \emph{diverse} if we want to quickly get good FFA approximations.
Due to the \emph{incremental} operation of the minimal hitting set enumeration algorithms, this is hard to achieve if we \emph{target} AXp enumeration.
But if we aim for CXp's then we can extract diverse AXp's by duality, which helps us get reasonable FFA approximations quickly converging to the exact FFA values.

Nevertheless, our experiments with the setup of~\cite{yis-corr23} suggest that AXp enumeration in fact tends to finish much earlier than CXp enumeration despite ``losing'' at the beginning.
This makes one wonder what to opt for if good and quickly converging FFA approximation is required: AXp enumeration or CXp enumeration. 
On the one hand, the latter quickly gives a large set of diverse AXp's and good initial FFA approximations.
On the other hand, complete AXp enumeration finishes much faster, i.e. exact FFA is better to compute by targeting AXp's.

Motivated by this, we propose to set up MARCO in a way that it starts with CXp enumeration and then seamlessly switches to AXp enumeration using two simple heuristic criteria.
It should be first noted that to make efficient switching in the target explanations, we employ pure SAT-based hitting set enumerator, where an incremental SAT solver is called multiple times aiming for minimal or maximal models~\cite{giunchiglia-ecai06}, depending on the phase.
This allows us to keep all the explanations found so far without ever restarting the hitting set enumerator.

As we observe that AXp's are normally larger than CXp's, both criteria for switching the target build on the use of the average \emph{size} of the last $w$ AXp's and the last $w$ CXp's enumerated in the most recent iterations of the MARCO algorithm.
(Recall that our MARCO setup aims for subset-minimal explanations rather than cardinality-minimal explanations, i.e. neither target nor dual explanations being enumerated are cardinality-minimal.)
This can be seen as inspecting ``sliding windows'' of size $w$ for both AXp's and CXp's.
In particular, assume that the sets of the last $w$ AXp's and CXp's are denoted as $\axps^w$ and $\cxps^w$, respectively.
First, switching can be done as soon as we observe that CXp's on average are \emph{much} smaller than AXp's, i.e. when
\begin{align}\label{eq:cond1}
    \frac{\sum_{\fml{X}\in\axps^w}{|\fml{X}|}}{\sum_{\fml{Y}\in\cxps^w}{|\fml{Y}|}} \geq \alpha,
\end{align}
where $\alpha\in\mbb{R}$ is a predefined numeric parameter.
The rationale for this heuristic is as follows.
Recall that extraction of a subset-minimal dual explanation is done by deciding the validity of the corresponding predicate, either \eqref{eq:axp} or \eqref{eq:cxp}, while iteratively removing features from the candidate feature set (see \Cref{sec:ffa}).
As such, if the vast majority of CXp's is small, their extraction leads to the lion's share of the decision oracle calls being \emph{satisfiable}.
On the contrary, extracting large AXp's as dual explanations leads to most of the oracle calls proving \emph{unsatisfiability}.
Hence, we prefer to deal with cheap satisfiable calls rather than expensive unsatisfiable ones.
Second, we can switch when the average CXp size ``stabilizes''.
Here, let us denote a new CXp being just computed as $\fml{Y}_{\text{new}}$.
Then the second criterion can be examined by checking if the following holds:
\begin{align}\label{eq:cond2}
    \left|\left|\fml{Y}_{\text{new}}\right|-\frac{\sum_{\fml{Y}\in\cxps^w}{|\fml{Y}|}}{w}\right| \leq \varepsilon,
\end{align}
with $\varepsilon\in\mbb{R}$ being another numeric parameter.
This condition is meant to signify the point when the set of AXp's we have already accumulated is diverse enough for all the CXp's to be of roughly equal size, which is crucial for good FFA approximations.
Overall, the switching can be performed when either of the two conditions \eqref{eq:cond1}--\eqref{eq:cond2} is satisfied.

\input{algs/enum-switch}

\Cref{alg:enum-switch} shows the pseudo-code of the adaptive explanation enumeration algorithm.
Additionally to the classifier's representation $\kappa$, instance $\mbf{v}$ to explain and its class $c$, it receives 3 numeric parameters: window size $w\in\mbb{N}$ and switching-related constants $\alpha,\varepsilon\in\mbb{R}$.
The set of CXp's (resp. AXp's) is represented by $\exps_0$ (resp. $\exps_1$) while the target phase of the hitting set enumerator is denoted by $\rho\in\{0,1\}$, i.e. at each iteration \Cref{alg:enum-switch} aims for $\exps_\rho$ explanations.
As initially $\rho=0$, the algorithm targets CXp enumeration.
Each of its iterations starts by computing a minimal hitting set $\mu$ of the set $\exps_{1-\rho}$ subject to $\exps_\rho$ (see \cref{ln:mhs}), i.e. we want $\mu$ to be a hitting set of $\exps_{1-\rho}$ different from all the target explanations in $\exps_\rho$ found so far.
If no hitting set exists, the process stops as we have enumerated all target explanations.
Otherwise, each new $\mu$ is checked for being a target explanation, which is done by invoking a reasoning oracle to decide the validity either of \eqref{eq:axp} if we target AXp's, or of \eqref{eq:cxp} if we target CXp's.
If the test is positive, the algorithm records the new explanation $\mu$ in $\exps_\rho$.
Otherwise, using the standard apparatus of formal explanations, it extracts a subset-minimal dual explanation $\nu$ from the complementary set $\fml{F}\setminus\mu$, which is then recorded in $\exps_{1-\rho}$.
Each iteration is additionally augmented with a check whether we should switch to the opposite phase $1-\rho$ of the enumeration.
This is done in \cref{ln:swcond} by testing whether at least one of the conditions \eqref{eq:cond1}--\eqref{eq:cond2} is satisfied.

\paragraph{Remark.}
Flipping enumeration phase $\rho$ can be seamlessly done because we apply pure SAT-based hitting enumeration~\cite{giunchiglia-ecai06} where both $\exps_\rho$ and $\exps_{1-\rho}$ are represented as sets of \emph{negative} and \emph{positive} blocking clauses, respectively.
As such, by instructing the SAT solver to opt for minimal or maximal models,\footnote{In SAT solving, a \emph{minimal} model is a satisfying assignment that respects subset-minimality wrt. the set of positive literals, i.e. where none of the 1's can be replaced by a 0 such that the result is still a satisfying assignment~\cite{sat-handbook21}. \emph{Maximal} models can be defined similarly wrt. subset-minimality of negative literals.} we can flip from computing hitting sets of $\exps_{1-\rho}$ subject to $\exps_\rho$ to computing hitting sets of $\exps_\rho$ subject to $\exps_{1-\rho}$, and vice versa.
Importantly, this can be done while incrementally keeping the internal state of the SAT solver, i.e. no learnt information gets lost after the phase switch.
Also, note that although the algorithm allows us to apply phase switching multiple times, our practical implementation switches \emph{once} because AXp enumeration normally gets done much earlier than CXp enumeration, i.e. there is no point in switching back.


%% file: algs/enum-switch.tex

\begin{algorithm}[t]
  \begin{algorithmic}[1]
    \Procedure{AdaptiveXpEnum}{$\kappa$, $\mbf{v}$, $c$, $w$, $\alpha$, $\varepsilon$}
      \State $\left(\exps_0, \exps_1\right)\gets (\emptyset, \emptyset)$ \label{ln:init}
      \Comment{CXp's and AXp's to collect}
      \State $\rho \gets 0$
      \Comment{Target phase of enumerator, initially CXp}
      \While{true}
        \State $\mu\gets\Call{MinimalHS}{\exps_{1-\rho}, \exps_\rho, \rho}$ \label{ln:mhs}
        \IfThen{$\mu=\bot$}{\Break} \label{ln:nocand}
        \If{$\Call{IsTargetXp}{\mu,\kappa,\mbf{v},c}$} \label{ln:check}
          \State $\exps_\rho\gets\exps_\rho\cup\{\mu\}$ \label{ln:targrec}
          \Comment{Collect target expl. $\mu$}
        \Else
          \State $\nu\gets\Call{ExtractDualXp}{\fml{F}\setminus\mu,\kappa,\mbf{v},c}$ \label{ln:extract}
          \State $\exps_{1-\rho}\gets\exps_{1-\rho}\cup\{\nu\}$ \label{ln:dualrec}
          \Comment{Collect dual expl. $\nu$}
        \EndIf
        \If{$\Call{IsSwitchNeeded}{\exps_\rho,\exps_{1-\rho}, w, \alpha, \varepsilon}$} \label{ln:swcond}
            \State $\rho\gets 1-\rho$ \label{ln:switch}
            \Comment{Flip phase of \textsc{MinimalHS}}
        \EndIf
      \EndWhile
      \Return{$\exps_1$, $\exps_0$}
      \Comment{Result AXp's and CXp's}
    \EndProcedure
  \end{algorithmic}
  \caption{Adaptive Explanation Enumeration}
  \label{alg:enum-switch}
\end{algorithm}

%% file: res.tex
\section{Experimental Results} \label{sec:res}
We evaluate our approach to FFA approximation for 
gradient boosted trees~(BTs) on various data
using a range of metrics.

\subsection{Experimental Setup}\label{sec:setup}
The experiments were done on an Intel Xeon 8260 CPU running Ubuntu 20.04.2 LTS, with the 8GByte memory 
limit. 

\paragraph{Prototype Implementation.}
The proposed approach was prototyped as a set 
of Python scripts, building on the approach of~\cite{yis-corr23}. 
%
The implementation can be found in the supplementary materials.
The proposed approach is referred to as \mswitch, where
the MARCO algorithm switches from CXp to AXp enumeration
based on the conditions~\eqref{eq:cond1}--\eqref{eq:cond2}.
For this, ``sliding windows" of size~$w = 50$ are used; $\varepsilon=2$ in~\eqref{eq:cond1} to signify the extent by which the size 
of AXp's should be larger than the size of CXp's, while
$\alpha=1$ in~\eqref{eq:cond2} denoting the
stability of the average CXp size.

\paragraph{Datasets and Machine Learning Models.}
The experiments include five well-known image and text datasets.
We use the widely used \emph{MNIST} ~\cite{deng2012mnist,pytorch-neurips19}
dataset, which features hand-written digits 
from 0 to 9, with two concrete binary
classification problems created: 1~vs.~3 and 1~vs.~7.
Also, we consider the image dataset \emph{PneumoniaMNIST}~\cite{medmnistv2} differentiating normal X-ray cases from pneumonia.
Since extracting \emph{exact} FFA values for aforementioned image datasets turns out
to be hard~\cite{yis-corr23}, we perform a size reduction,
downscaling these images from $\text{28} \times \text{28} \times \text{1}$ to $\text{10} \times
\text{10} \times \text{1}$.
Additionally, 2 text datasets are considered in the experiments: \emph{Sarcasm}~\cite{misra2023Sarcasm,misra2021sculpting} and
\emph{Disaster}~\cite{hdcg-kaggle19}.
The \emph{Sarcasm} dataset contains news headlines collected 
from websites, along with binary labels indicating whether 
each headline is sarcastic or non-sarcastic.
The \emph{Disaster} dataset consists of the contents 
of tweets with labels about whether a user 
announces a real disaster or not.
The 5 considered datasets are randomly divided into 80\% training
and 20\% test data.
To evaluate the performance of the proposed approach, 15 test instances in each test set are randomly selected. 
Therefore, the total number of instances used in the experiments is 75.
In the experiments, gradient boosted trees~(BTs) are trained by
XGBoost~\cite{guestrin-kdd16a}, where each BT consists of 25 to 40 trees of
depth 3 to 5 per class.\footnote{Test accuracy for \emph{MNIST}, \emph{PneumoniaMNIST}, \emph{Sarcasm}, and \emph{Disaster} datasets is 0.99, 0.83, 0.69, and 0.74, respectively.}
%

\paragraph{Competitors and Metrics.}
We compare the proposed approach~(\mswitch) against the original MARCO algorithms targeting AXp~(\maxp) or CXp~(\mcxp) enumeration.
We evaluate the FFA generated by these 
approaches by comparing it to the 
\emph{exact} FFA through a variety of metrics,
including errors, Kendall's Tau~\cite{kendall1938}, rank-biased overlap~(RBO)~\cite{wmz-tois10}, 
and Kullback–Leibler~(KL) divergence~\cite{kl-ams51}.
The \emph{error} is quantified using Manhattan distance,
i.e. the sum of absolute differences across 
all features in an instance.
The comparison of feature ranking is assessed
by Kendall's Tau and RBO,
while feature distributions are compared by 
KL divergence.\footnote{Kendall's Tau is a correlation coefficient
metric that evaluates the ordinal association
between two ranked lists,
providing a similarity measure for the order of values,
while RBO quantifies similarity between two ranked lists,
considering both the order and depth of the overlap. KL-divergence measures 
the statistical distance between two
probability distributions.}
Kendall's Tau and RBO produces values within
the range of $[-\text{1}, \text{1}]$ and
$[\text{0}, \text{1}]$, respectively, where higher values
in both metrics indicate stronger agreement or similarity 
between the approximate FFA and the exact FFA.
KL-divergence ranges from  0 to $\infty$, with
the value approaching 0 reflecting  better alignment
between approximate FFA distribution and 
the exact FFA distribution.
Note that if a feature in the exact FFA distribution holds a non-zero probability but is assigned a zero probability in the approximate one, the KL value becomes $\infty$.
Finally, we also compare the efficiency of generating
AXp's in the aforementioned approaches.

\begin{figure}[!t]
	\centering
    \includegraphics[width=0.45\textwidth]{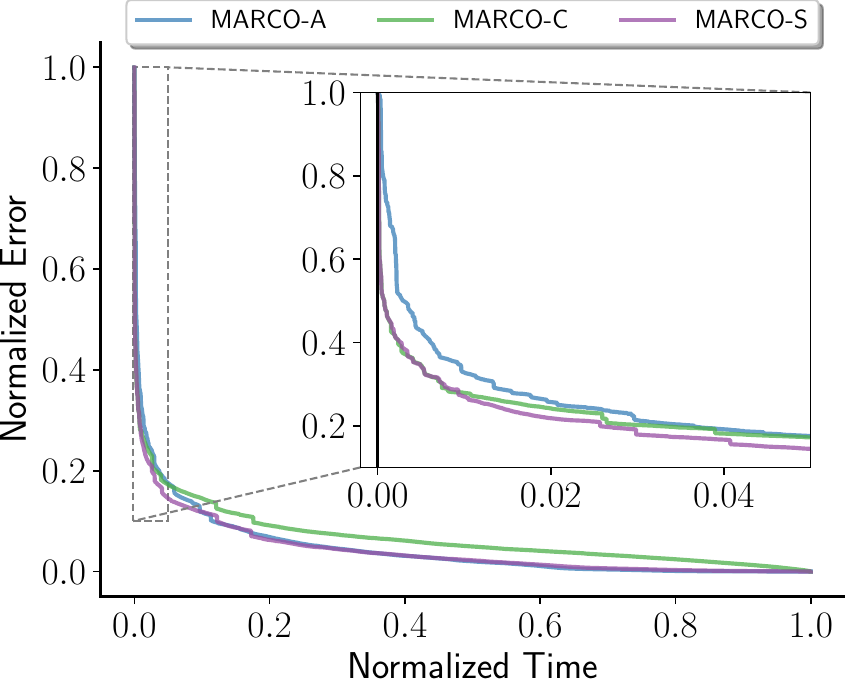}
    \caption{FFA error over time.} 
    \label{fig:error}
\end{figure}

\begin{figure}[!t]
	\centering
    \includegraphics[width=0.45\textwidth]{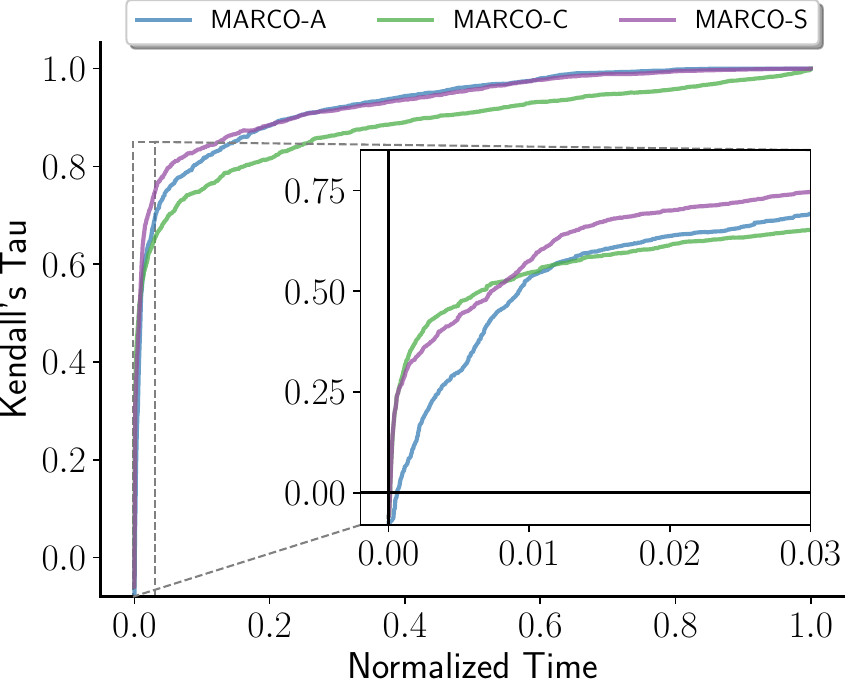}
    \caption{Kendall's Tau over time.}
    \label{fig:tau}
     \end{figure}
     
\begin{figure}[!t]
	\centering
    \includegraphics[width=0.45\textwidth]{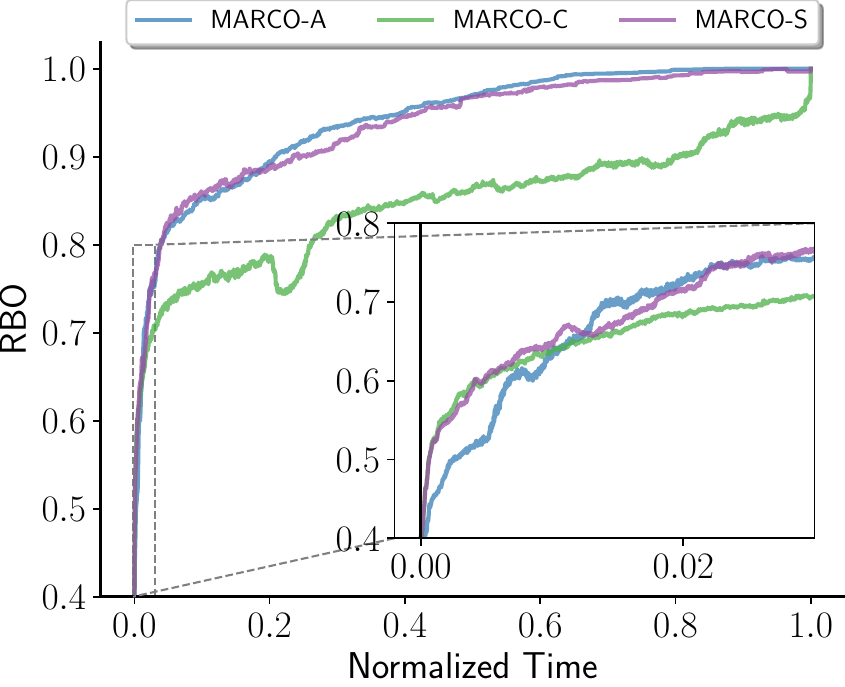}
    \caption{RBO over time.}
    \label{fig:rbo}
\end{figure}

\begin{figure}[!t]
	\centering
    \includegraphics[width=0.45\textwidth]{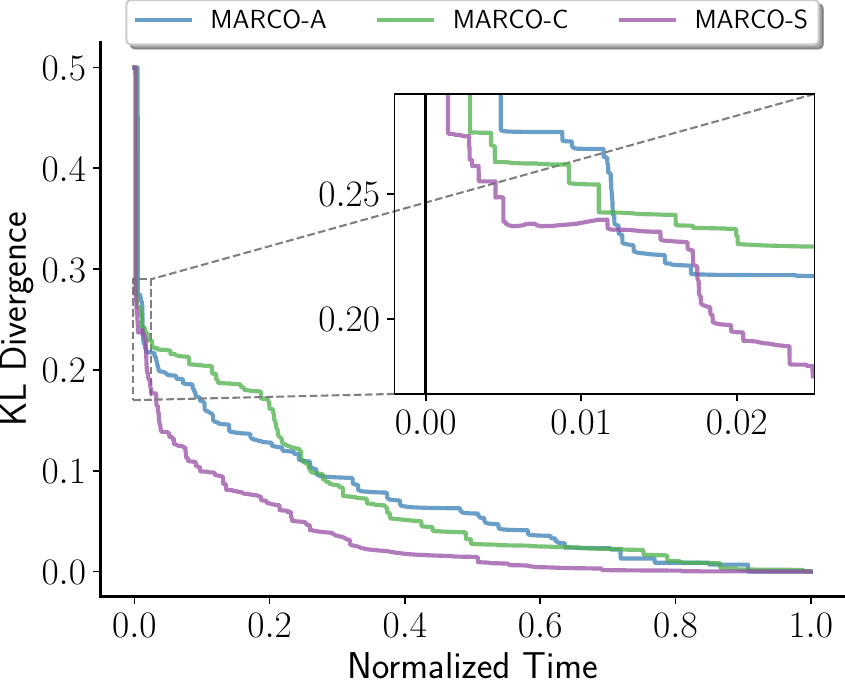}
    \caption{KL divergence over time.}
    \label{fig:kl}
\end{figure}

\begin{figure}[!t]
	\centering
    \includegraphics[width=0.45\textwidth]{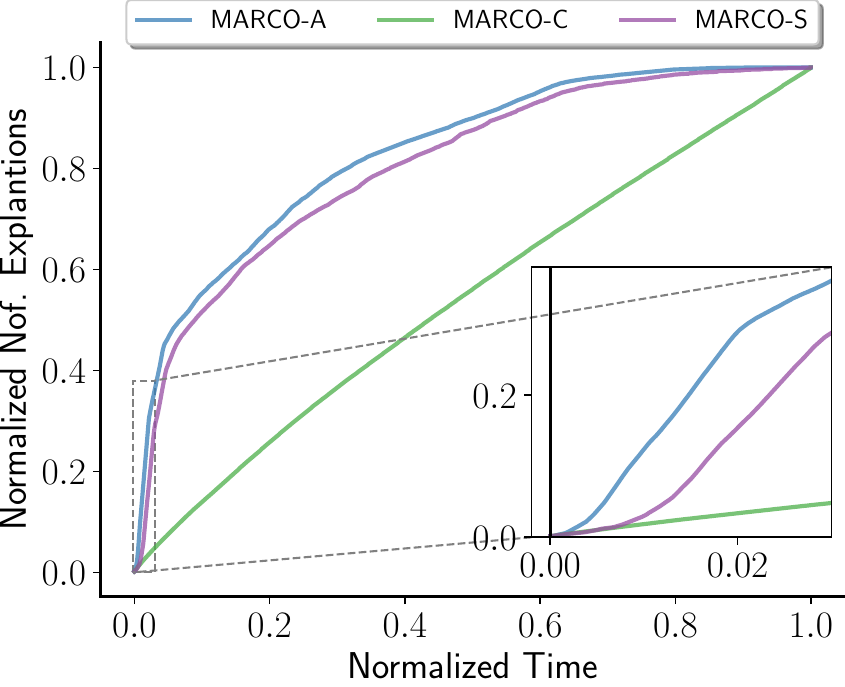}
    \caption{Number of explanations over time.}
    \label{fig:expls}
\end{figure}

\subsection{Integrated Results} \label{sec:resint}
This section compares the proposed approach against 
the original MARCO algorithms for both AXp enumeration
and CXp enumeration within the examined datasets.
Figures~\ref{fig:error} to \ref{fig:expls} 
illustrate the results of approximate 
FFA in terms of the five aforementioned metrics, namely,
errors, Kendall's Tau, RBO, KL divergence, and 
the number of AXp's.
These results are obtained by averaging 
values across all instances.
Note that since KL-divergence is $\infty$ when
there exists a feature in the exact FFA distribution
that holds a non-zero probability but is assigned a zero
probability in the approximate one, to address this issue
we assign 0.5 as the KL-divergence value instead of $\infty$
in this case.\footnote{According the experimental results
we obtained, the maximum of non-infinity KL-divergence 
values is not greater than 0.5.}
The average runtime to extract exact FFA is 
3255.30s~(from 2.15s to 29191.42s),
19311.87s~(from 9.39s to 55951.57s), 
and 3509.50s~(from 9.26s to 30881.55s)
for \maxp, \mcxp, and \mswitch, respectively.
Since the runtime required to get exact FFA 
varies, we normalized the runtime in each instance into [0, 1],
where the longest time across three compared approaches 
in each instance is normalized to 1.
%
Furthermore, we normalized the number of AXp's in each instance
into the interval of [0, 1].
Finally, errors are also normalized into [0, 1] for each instance.
Detailed experimental results can be found in the supplementary material.

\paragraph{Errors.}
Figure~\ref{fig:error} displays the average errors 
of approximate FFA across all instances over time. 
In the early period, \mcxp obtained more accurate 
approximate FFA in terms of errors compared with \maxp, 
while beyond the 0.04 time fraction, the latter  
surpasses the former and eventually achieves 0 error
faster, which also indicates that \maxp requires less
time to acquire the exact FFA.
Motivated by the above observation, the proposed approach
aims at replicating the ``best of two worlds'' during the FFA approximation process.
Observe that \mswitch commences with
the MARCO algorithm targeting CXp's
and so replicates the superior behavior of \mcxp 
during the initial stage.
Over time, \mswitch triggers a switch criterion 
and transitions to targeting AXp's, thus
emulating the behavior of the better competitor 
beyond the early stage, i.e. \maxp.
Finally, \mswitch acquires FFA with 0 error~(i.e. 
exact FFA) as efficiently as \maxp.

\paragraph{Feature Ranking.} 
The results of Kendall's Tau and RBO are depicted
in Figures~\ref{fig:tau}--\ref{fig:rbo}.
Initially, \mcxp outperforms \maxp in
terms of both feature ranking metrics.
As time progresses, \maxp starts to surpass \mcxp
since 0.01 time fraction until the point of acquiring 
the exact FFA.
Figures~\ref{fig:tau}--\ref{fig:rbo} demonstrate that initially
\mswitch manages to keep close to the better performing
\mcxp.
When \maxp starts dominating, \mswitch switches 
the target phase from CXp's to AXp's, 
replicating the superior performance displayed by \maxp.

\paragraph{Distribution.}
\Cref{fig:kl} depicts the average results of
KL divergence over time.
Similar to feature ranking, \mcxp is initially capable
of computing an FFA distribution closer to 
the exact FFA distribution.
Beyond the initial stage, \maxp exhibits the ability
to generate closer FFA distribution.
Once again, \mswitch replicates the superior behavior between
\maxp and \mcxp in most of time.
During the initial stage, \mswitch reproduces the 
behavior of \mcxp, and switch to target
AXp's directly when the switch criterion is met.
Surprisingly, \mswitch \emph{outperforms both competitors} 
throughout the entire time interval.

\paragraph{Number of AXp's.}
The average results of the normalized number of AXp's 
are illustrated in Figure~\ref{fig:expls}.
\maxp generates AXp's faster and finishes earlier than \mcxp.
%
Observe that the proposed
approach~\mswitch is able to avoid the inferior performance 
between \maxp and \mcxp throughout the process.
Initially, \mswitch replicates the behavior of \mcxp
and then switches to target AXp's 
to replicate the performance of \maxp.

\paragraph{Summary.}
\mswitch can replicate the behavior of the superior competitor for most of the computation duration, leading to efficient and good approximation of FFA.
As illustrated by Figures~\ref{fig:error}--\ref{fig:kl} in terms of FFA errors, Kendall's Tau, RBO, 
and KL divergence, starting from CXp enumeration and switching to AXp enumeration based on the criteria \eqref{eq:cond1}--\eqref{eq:cond2} successfully replicates the behavior of the winning configuration of MARCO, thus getting close of their \emph{virtual best solver}.
%
Although in terms of the number of AXp's shown in \Cref{fig:expls} \maxp consistently outperforms \mcxp, those AXp's are not diverse enough to allow \maxp to beat \mcxp in other relevant metrics.
This is alleviated by \mswitch, which manages to get enough diverse AXp's initially and then switches to  target AXp's to catch up with the performance of \maxp.

%% file: conc.tex
\section{Conclusions} \label{sec:conc}

Formal feature attribution (FFA) defines a crisp and easily understood notion of feature importance to a decision.
Unfortunately for many classifiers and datasets it is challenging to compute exactly. As our paper demonstrates, it remains hard even if the set of CXp's is provided. Hence there is a need for \emph{anytime} approaches to compute FFA. Surprisingly, using CXp enumeration to generate AXp's leads to fast good approximations of FFA, but in the longer term it is worse than simply enumerating AXp's. This paper shows how to combine the approaches by diligently switching the phase of enumeration, without losing information computed in the underlying MARCO enumeration algorithm. 
This gives a highly practical approach to computing FFA.

The proposed mechanism can be adapted to a multitude of other problems, e.g. in the domain of over-constrained systems, where one wants to collect a \emph{diverse} set of minimal unsatisfiable subsets (MUSes) as the same minimal hitting set duality exists between MUSes and
minimal correction sets (MCSes)~\cite{birnbaum-jetai03,reiter-aij87}.


%% file: appendix.tex
\newpage
\clearpage
\appendix
\appendixpage

\section{Detailed Experimental Results}

This appendix compares the proposed approach~(\mswitch)
against the original MARCO algorithms for targeting AXp's~(\maxp)
and CXp's~(\mcxp) in each considered dataset, 
namely \emph{MNIST-1vs3}, \emph{MNIST-1vs7}, \emph{PneumoniaMNIST}, \emph{Sarcasm}, and \emph{Disaster}.
Figures~\ref{fig:error-dt} to \ref{fig:expls-dt} depict 
the average results of the comparison between the approximate FFA
and the exact FFA using 5 aforementioned metrics, 
namely, errors, Kendall's Tau, RBO, KL divergence, and
the number of AXp's.
The results are acquired by averaging the values across 
15 selected instances in a dataset.
The average runtime of the three approaches to 
acquire the exact FFA in each datadset is
summarized in Table~\ref{tab:rtime-dt}.
Once again, as KL-divergence is $\infty$ when a feature 
in the exact FFA distribution holds a non-zero probability 
but is assigned a zero probability in the approximate
one, we assign a value of 0.5 to KL-divergence
in stead of $\infty$.
Note that the maximum of non-infinity KL-divergence
values is not greater than 0.5 in the experimental results we
acquired.
The same normalization technique is applied to
runtime, errors, and the number of AXp's as described in Section~\ref{sec:resint}.
We normalized these three metrics across the three 
compared approaches in each instance into [0, 1]. 

\paragraph{Errors.}
The average errors of approximate FFA across selected instances
in each dataset over time are shown in Figure~\ref{fig:error-dt}.
In the initial stage, compared with \maxp, 
\mcxp achieves a more accurate approximation of
FFA in terms of errors in \emph{MNIST-1vs3}, \emph{MNIST-1vs7},
\emph{PneumoniaMNIST}, and \emph{Disaster} datasets
while beyond the early phase \maxp starts
to outperform \mcxp and eventually achieves 
0 error faster in these datasets, indicating that \maxp is able to
spend less time to obtain the exact FFA.
However, \maxp consistently outperforms \mcxp in
the \emph{Sarcasm} dataset.
Inspired by these observations, the proposed approach
targets replicating the ``best of two world" in the process of approximating FFA.
Observe that \mswitch starts with the MARCO algorithm
of CXp enumeration and thus emulates the better behavior
displayed by \mcxp during the initial phase, as shown in Figures~\ref{fig:error-mnist1v3},
\ref{fig:error-mnist1v7}, \ref{fig:error-pneumonia}, and \ref{fig:error-disaster}.
Over time, \mswitch meets a switch criterion and transitions
to targeting AXp enumeration, so replicating the
behavior of the superior competitor beyond the initial stage.
As for the \emph{Sarcasm} dataset in Figure~\ref{fig:error-sarcasm},
\mswitch replicates the behavior of \mcxp and then 
switches to target AXp’s to replicate the performance of \maxp,
avoiding the inferior performance between \maxp and \mcxp.
Finally, in all datasets \mswitch is able to obtain FFA with 0 error~(i.e.
exact FFA) as efficient as \maxp.

\paragraph{Feature Ranking.}
Figures~\ref{fig:tau-dt} and \ref{fig:rbo-dt} illustrate
the results of Kendall’s Tau and RBO in each dataset.
Observe that \mcxp exhibits better performance than 
\maxp during the initial stage for both feature ranking metrics
in \emph{MNIST-1vs3}, \emph{MNIST-1vs7}, \emph{PneumoniaMNIST},
and \emph{Disaster} datasets.
Over time, \maxp gradually overtakes \mcxp until reaching
the point of obtaining the exact FFA in these datasets.
Nevertheless, in the \emph{Sarcasm} dataset, \maxp
consistently displays the superior performance 
during the entire period.
These figures demonstrate that \mswitch 
is capable of maintaining close to the superior performance
exhibited by \mcxp during the initial phase in all datasets 
except for \emph{Sarcasm}.
When \maxp begins to outperform \mcxp, \mswitch
transitions the target phase from CXp’s to AXp’s,
replicating the superior performance demonstrated by \maxp.
In the \emph{Sarcasm} dataset, switching from CXp enumeration
to AXp enumeration beyond the initial stage avoids 
reproducing the inferior performance between \maxp and \mcxp
in most of time.

\paragraph{Distribution.}
The average results of KL divergence over time
are depicted in Figure~\ref{fig:kl-dt}.
\mcxp is initially capable
of generating an FFA distribution more similar to 
the exact FFA distribution in 
\emph{MNIST-1vs3} and \emph{MNIST-1vs7} datasets.
Afterwards, \maxp exhibits the ability to compute FFA distribution
more similar to the exact FFA attribution.
However, \maxp consistently generate a closer FFA distribution 
than \mcxp in the other datasets.
Once again, \mswitch emulates the superior behavior between 
\maxp and \mcxp in most of time or avoids replicating the inferior
performance for a long time due to the switch mechanism.
\mswitch initially reproduces the behavior of \mcxp, and switches 
to target AXp’s when meeting the switch criterion.
Surprisingly, \mswitch exhibits the best performance among the competitors in most of the entire time interval in \emph{MNIST-1vs3} and \emph{MNIST-1vs7}.

\paragraph{Number of AXp's.}
Figure~\ref{fig:expls-dt} shows the normalized number
of AXp's in each dataset.
Observe that compared with \maxp, \mcxp is capable of generating AXp's
more efficiently during the early stage in \emph{MNIST-1vs3}
and \emph{MNIST-1vs7} datasets, but \maxp starts to outperform
\mcxp as time progresses.
In the other three datasets, \maxp achieves similar or better 
performance in the entire process.
As demonstrated by Figure~\ref{fig:expls-dt}, the proposed
approach~\mswitch is able to avoid the inferior performance 
between \maxp and \mcxp for most of the duration.
Initially, \mswitch emulates the behavior of \mcxp, 
and transitions to target AXp's to replicate the performance 
of \maxp afterwards, preventing the reproduction of
inferior performance.
Remarkably, in the \emph{MNIST-1vs7} dataset, 
\mswitch  emerges as the best-performing approach
for most of time.

\begin{table*}[t!]
\caption{Average runtime(s) in each dataset.}
\label{tab:rtime-dt}
\centering
\begin{tabular}{cccccc}  \toprule
 \textbf{Approach} & \textbf{MNIST-1vs3} & \textbf{MNIST-1vs7} & \textbf{Pneumoniamnist} & \textbf{Sarcasm} & \textbf{Disaster} \\ \midrule
 \textbf{\maxp} & 9350.79 & 2844.15 & 1972.41 & 669.91 & 1439.24 \\ 
 \textbf{\mcxp} & 14787.22 & 7412.40 & 8343.55 & 33391.29 & 32624.89 \\ 
 \textbf{\mswitch} & 9970.55 & 2959.15 & 2016.49 & 975.31 & 1626.01 \\ \bottomrule
\end{tabular}
\end{table*}

\begin{figure*}[!t]
    \centering
   \begin{subfigure}[b]{0.32\textwidth}
    \centering
    \includegraphics[width=\textwidth]{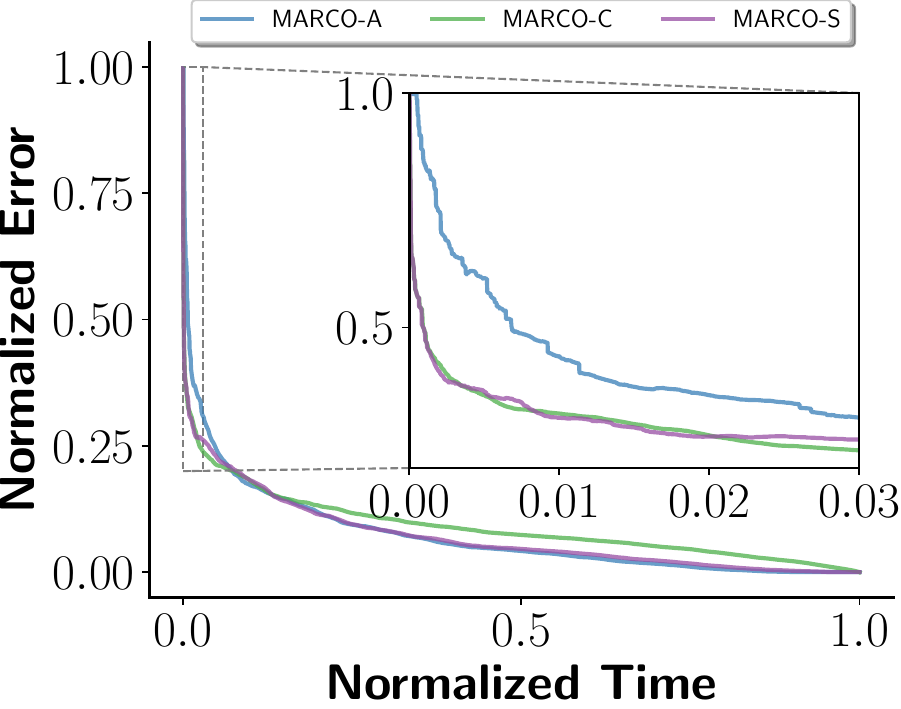}
    \caption{MNIST-1vs3}
    \label{fig:error-mnist1v3}
   \end{subfigure}
   \hfill
   \begin{subfigure}[b]{0.32\textwidth}
     \centering
     \includegraphics[width=\textwidth]{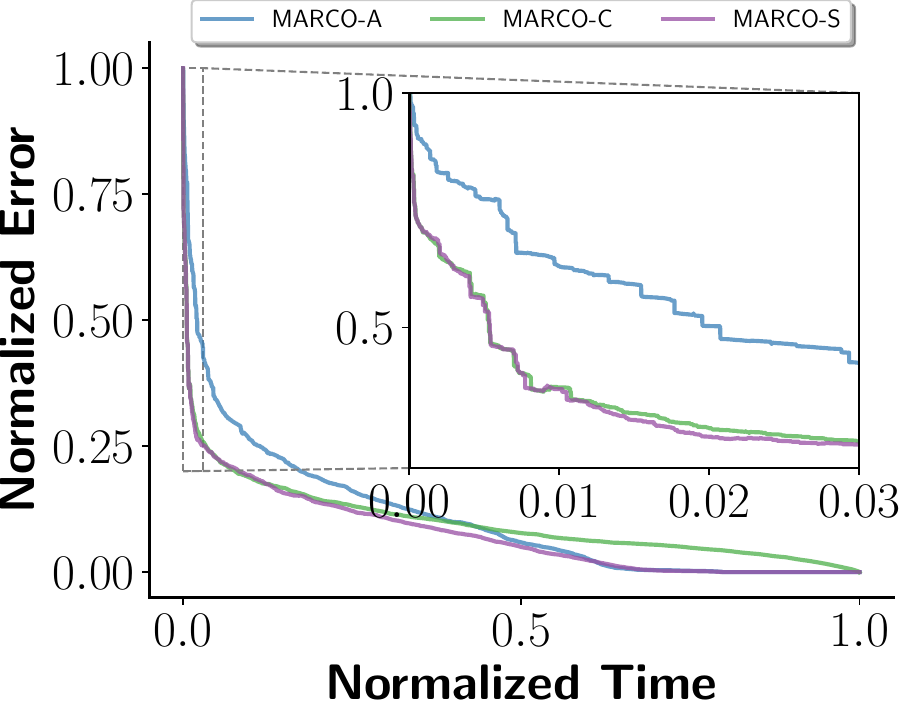}
     \caption{MNIST-1vs7}
    \label{fig:error-mnist1v7}
   \end{subfigure}
   \hfill
   \begin{subfigure}[b]{0.32\textwidth}
    \centering
    \includegraphics[width=\textwidth]{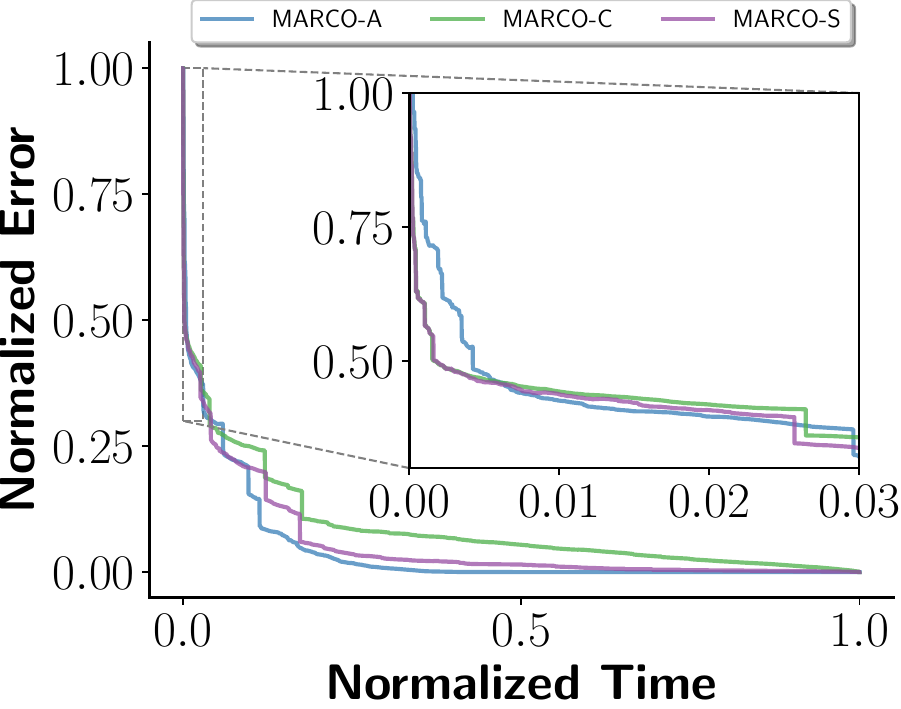}
    \caption{PneumoniaMNIST}
    \label{fig:error-pneumonia}
   \end{subfigure}
   \vspace{1em}
   
   \begin{subfigure}[b]{0.32\textwidth}
    \centering
    \includegraphics[width=\textwidth]{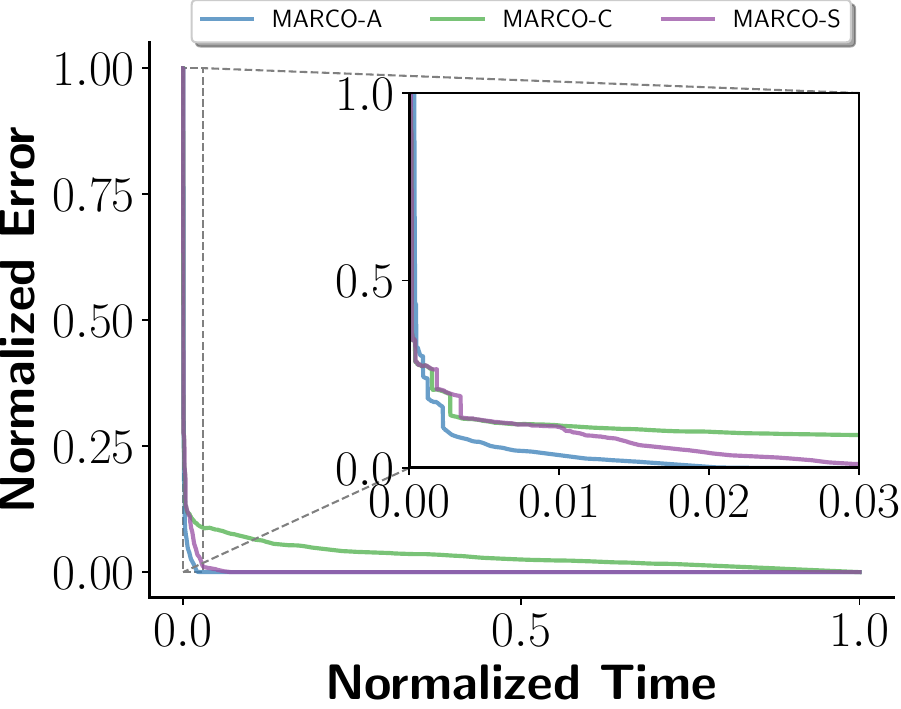}
    \caption{Sarcasm}
     \label{fig:error-sarcasm}
   \end{subfigure}%
   \hspace{1em}
   \begin{subfigure}[b]{0.32\textwidth}
     \centering
     \includegraphics[width=\textwidth]{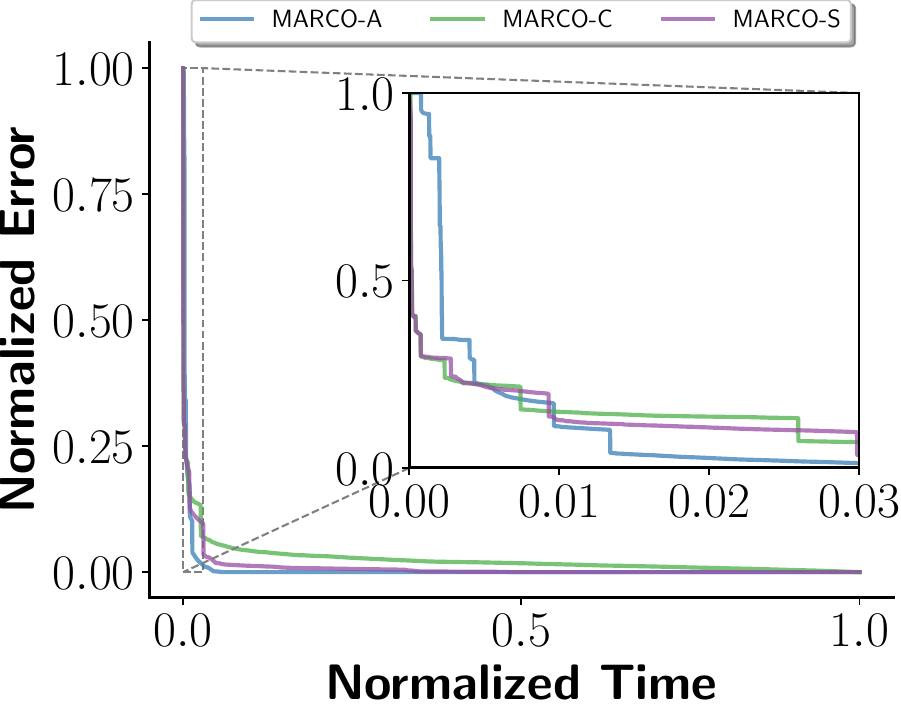}
     \caption{Disaster}
     \label{fig:error-disaster}
   \end{subfigure}
   \caption{FFA error over time in each dataset.}
   \label{fig:error-dt}
\end{figure*}

\paragraph{Summary.}
In alignment with the results presented in Section~\ref{sec:res},
\mswitch is able to replicate the behavior of the superior competitor
between \maxp and \mcxp throughout most of the computation
period, resulting in fast and good approximation of FFA. 
As depicted in Figures~\ref{fig:error-dt} to \ref{fig:kl-dt}, which display the results
of FFA errors, Kendall’s Tau, RBO, and KL divergence, beginning with CXp enumeration and subsequently transitioning to AXp enumeration based on criteria~\ref{eq:cond1}--\ref{eq:cond2} can reproduce the performance of the winning MARCO configuration, and therefore this approach enables the model to closely approach their virtual best solver.
Although \maxp consistently exhibits better performance than \mcxp
in \emph{PneumoniaMNIST}, \emph{Sarcasm} and \emph{Disaster} datasets
in terms of the number of AXp’s depicted in Figure~\ref{fig:expls-dt},
the lack of diversity among these AXp’s prevents \maxp
from outperforming \mcxp in other relevant metrics. 
This diversity issue is mitigated by \mswitch, which effectively addresses 
this by initially obtaining a sufficiently diverse set of AXp’s and subsequently transitioning to targeting AXp’s, thereby matching the performance of \maxp.

\begin{figure*}[!t]
  \centering
   \begin{subfigure}[b]{0.3\textwidth}
   \centering
    \includegraphics[width=\textwidth]{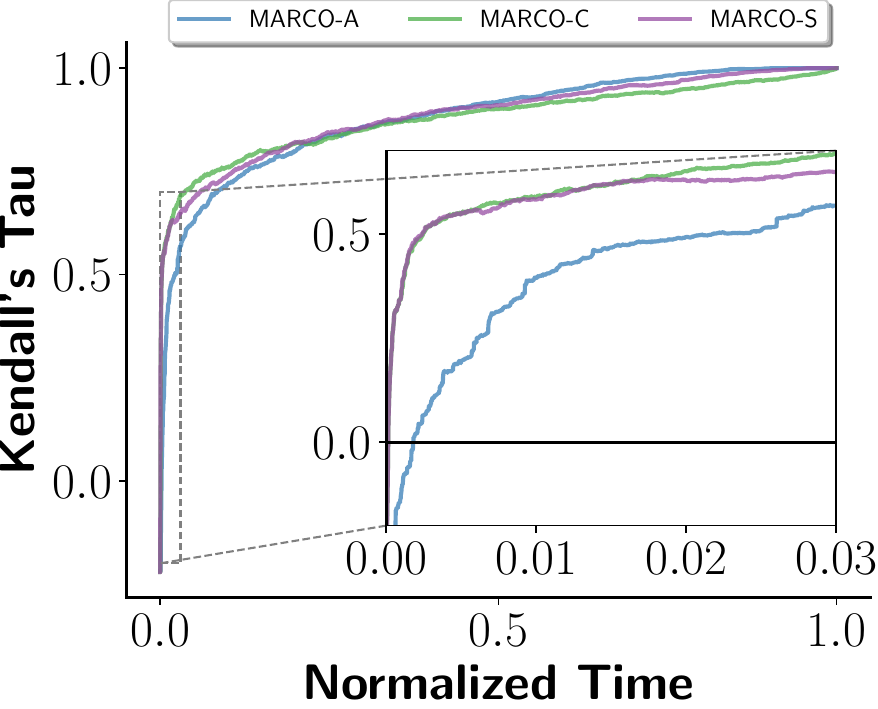}
    \caption{MNIST-1vs3}
    \label{fig:tau-mnist1v3}
   \end{subfigure}%
   \hfill
   \begin{subfigure}[b]{0.3\textwidth}
     \centering
     \includegraphics[width=\textwidth]{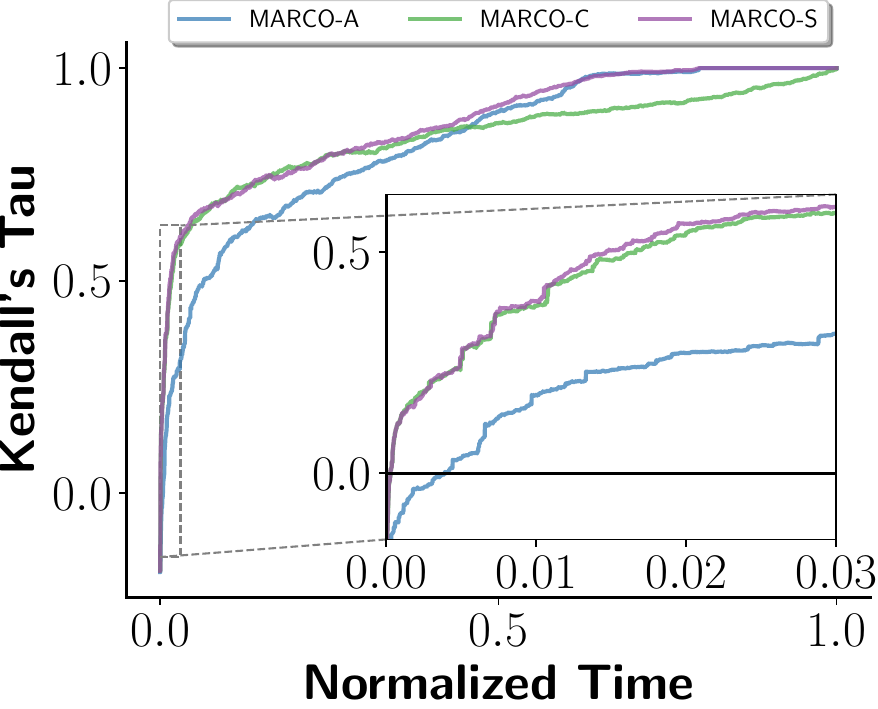}
     \caption{MNIST-1vs7}
     \label{fig:tau-mnist1v7}
   \end{subfigure}
   \hfill
   \begin{subfigure}[b]{0.3\textwidth}
    \centering
    \includegraphics[width=\textwidth]{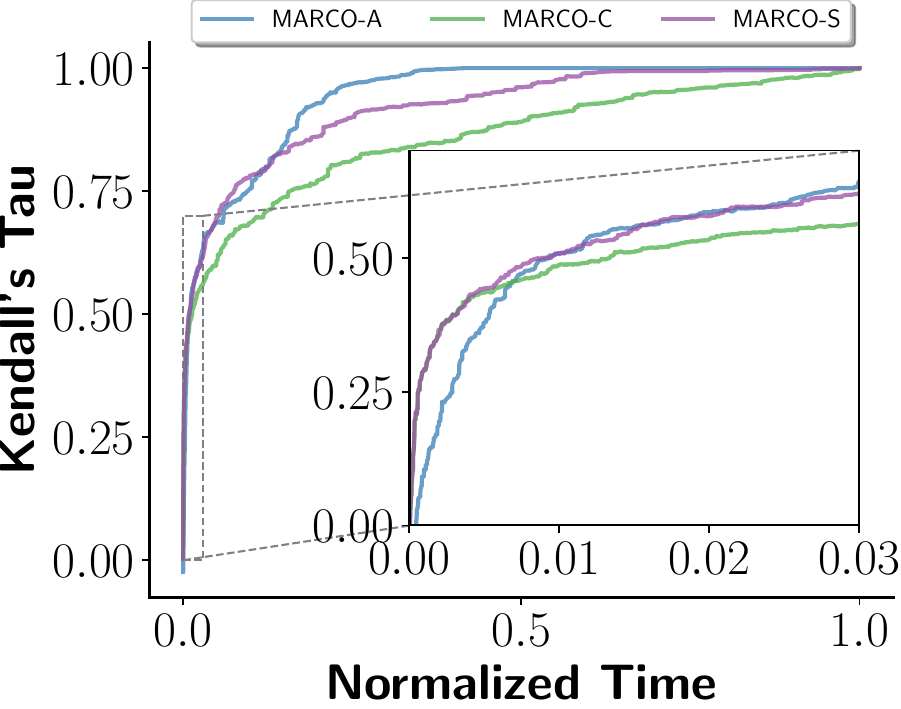}
    \caption{PneumoniaMNIST}
    \label{fig:tau-pneumonia}
   \end{subfigure}
    \vspace{1em}
   
   \begin{subfigure}[b]{0.3\textwidth}
    \centering
    \includegraphics[width=\textwidth]{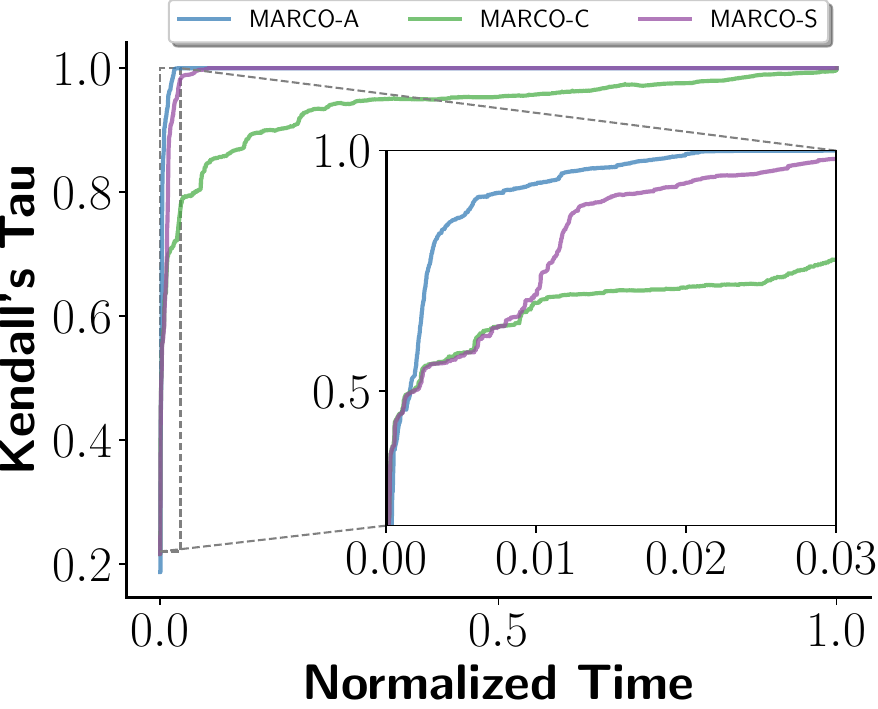}
    \caption{Sarcasm}
    \label{fig:tau-sarcasm}
   \end{subfigure}%
   \hspace{1em}
   \begin{subfigure}[b]{0.3\textwidth}
     \centering
     \includegraphics[width=\textwidth]{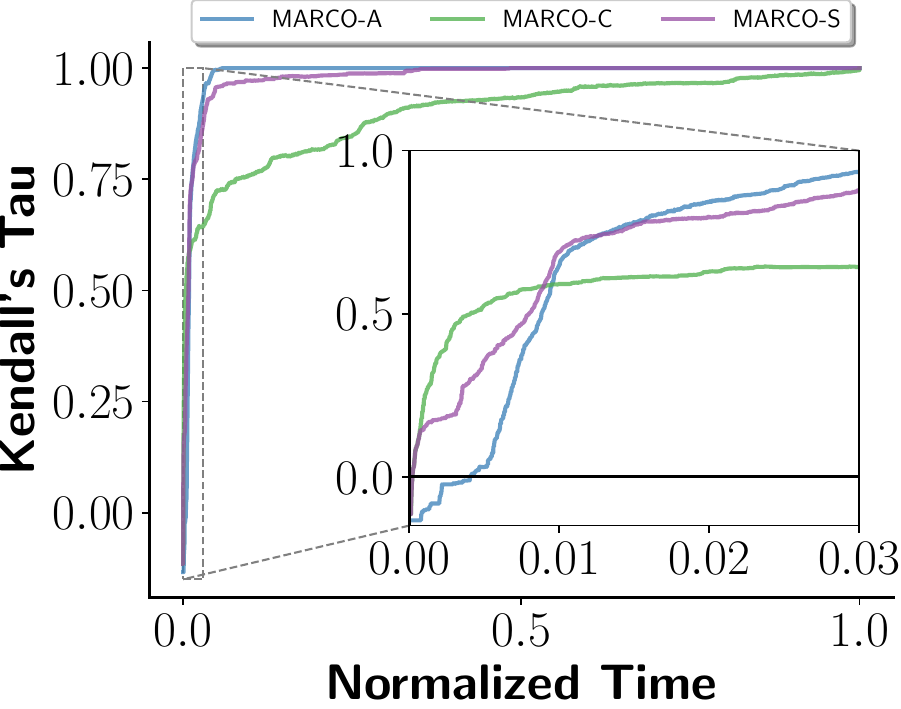}
     \caption{Disaster}
     \label{fig:tau-disaster}
   \end{subfigure}
   \caption{Kendall's Tau over time in each dataset.}
   \label{fig:tau-dt}
\end{figure*}

\begin{figure*}[!t]
  \centering
   \begin{subfigure}[b]{0.3\textwidth}
    \centering
    \includegraphics[width=\textwidth]{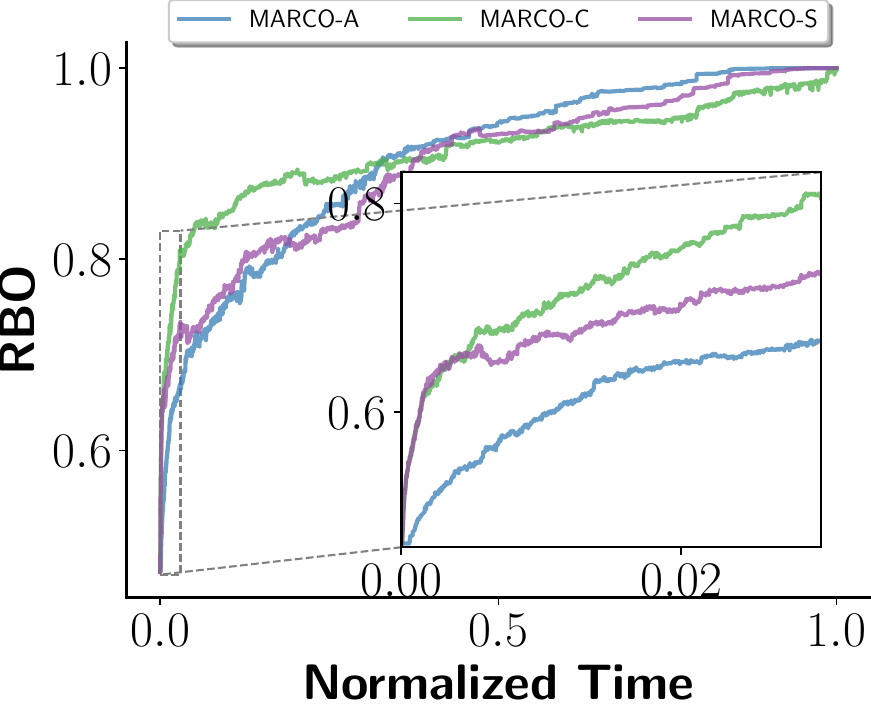}
    \caption{MNIST-1vs3}
    \label{fig:rbo-mnist1v3}
   \end{subfigure}%
   \hfill
   \begin{subfigure}[b]{0.3\textwidth}
     \centering
     \includegraphics[width=\textwidth]{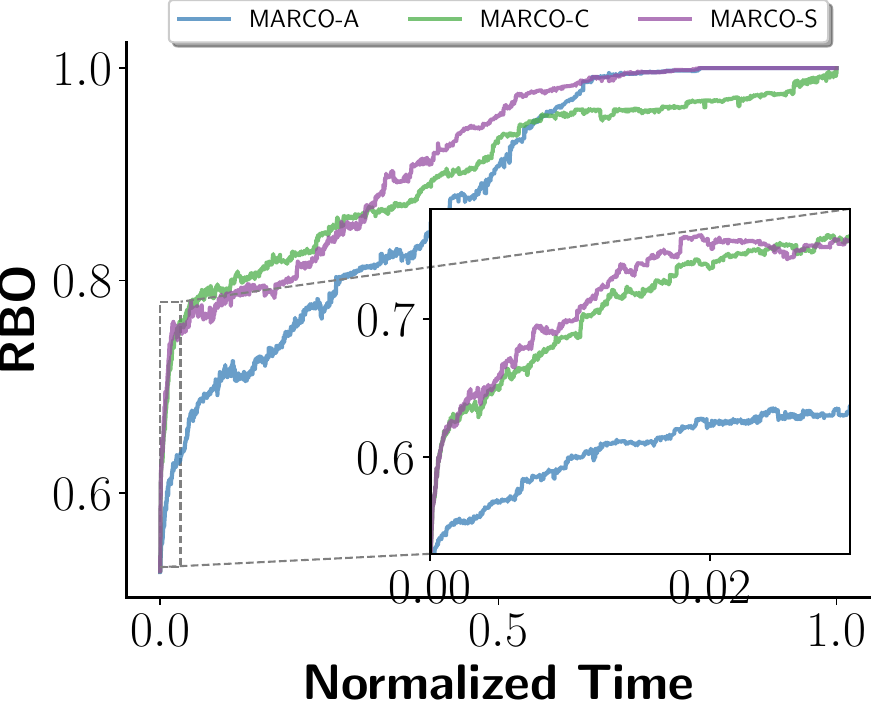}
     \caption{MNIST-1vs7}
     \label{fig:rbo-mnist1v7}
   \end{subfigure}%
   \hfill
   \begin{subfigure}[b]{0.3\textwidth}
    \centering
    \includegraphics[width=\textwidth]{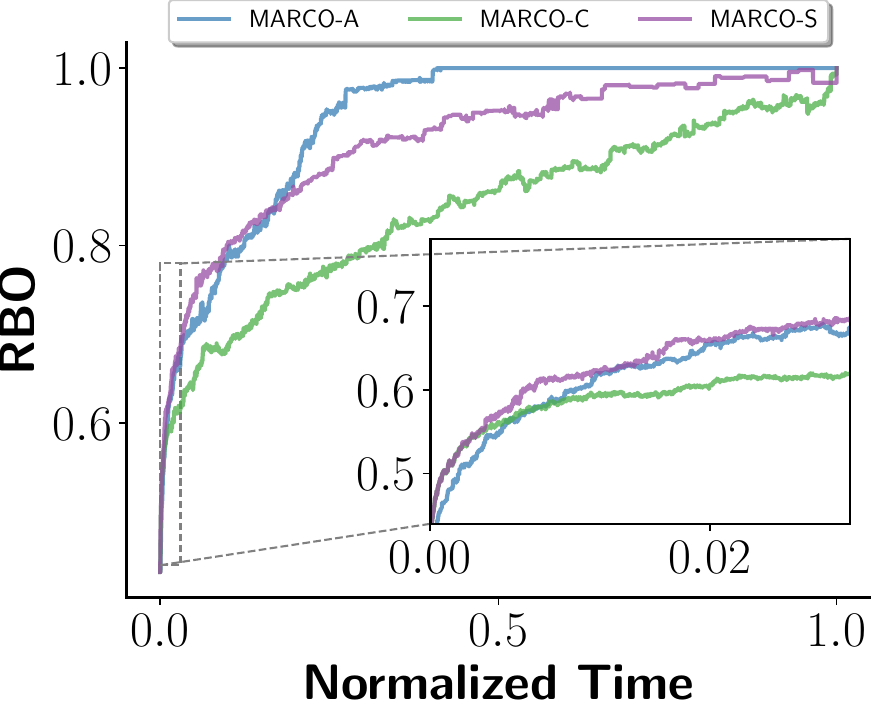}
    \caption{PneumoniaMNIST}
    \label{fig:fig:rbo-pneumonia}
   \end{subfigure}
   \vspace{1em}
   
   \begin{subfigure}[b]{0.3\textwidth}
    \centering
    \includegraphics[width=\textwidth]{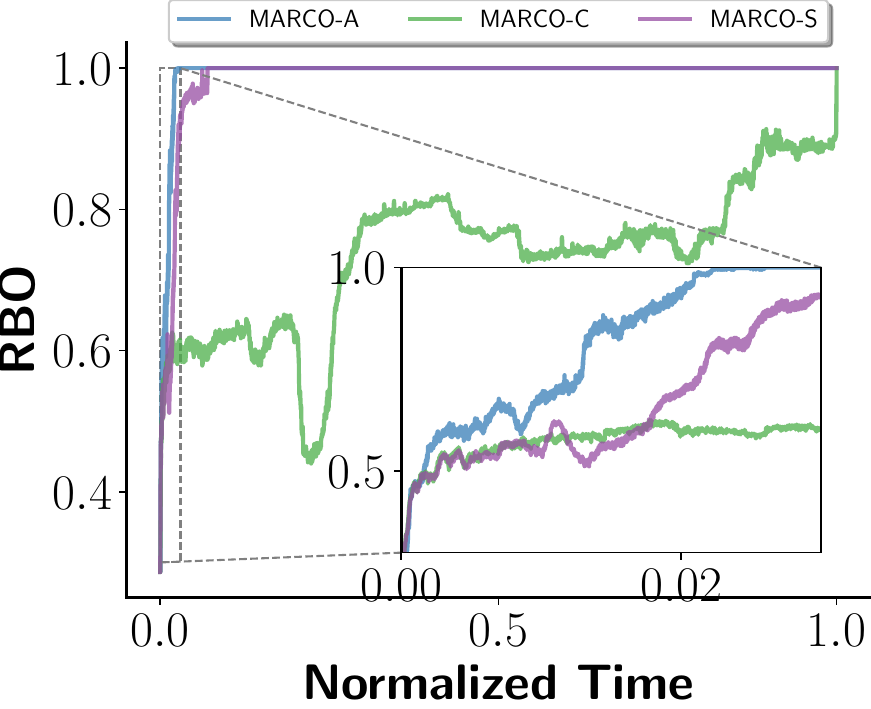}
    \caption{Sarcasm}
    \label{fig:rbo-sarcasm}
   \end{subfigure}%
   \hspace{1em}
   \begin{subfigure}[b]{0.3\textwidth}
     \centering
     \includegraphics[width=\textwidth]{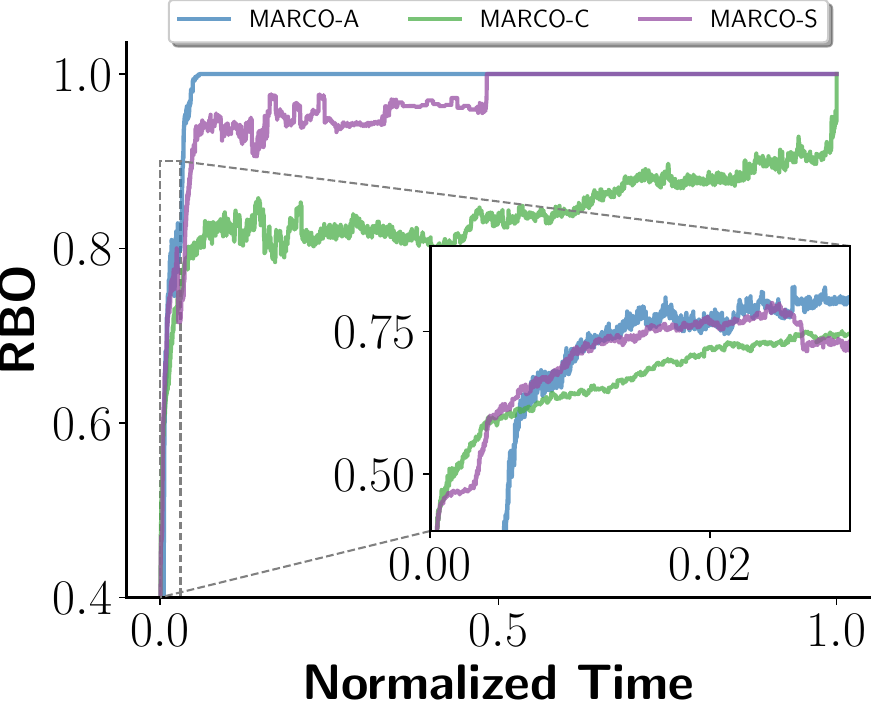}
     \caption{Disaster}
     \label{fig:rbo-disaster}
   \end{subfigure}
   \caption{RBO over time in each dataset.}
   \label{fig:rbo-dt}
\end{figure*}

\begin{figure*}[!t]
  \centering
   \begin{subfigure}[b]{0.3\textwidth}
    \centering
    \includegraphics[width=\textwidth]{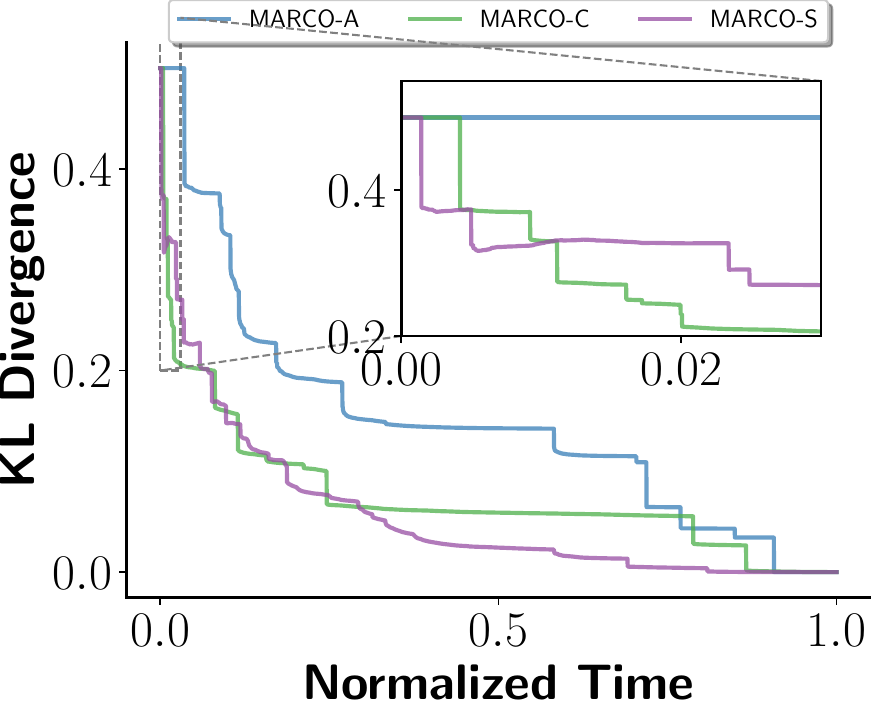}
    \caption{MNIST-1vs3}
    \label{fig:kl-mnist1v3}
   \end{subfigure}%
   \begin{subfigure}[b]{0.3\textwidth}
     \centering
     \includegraphics[width=\textwidth]{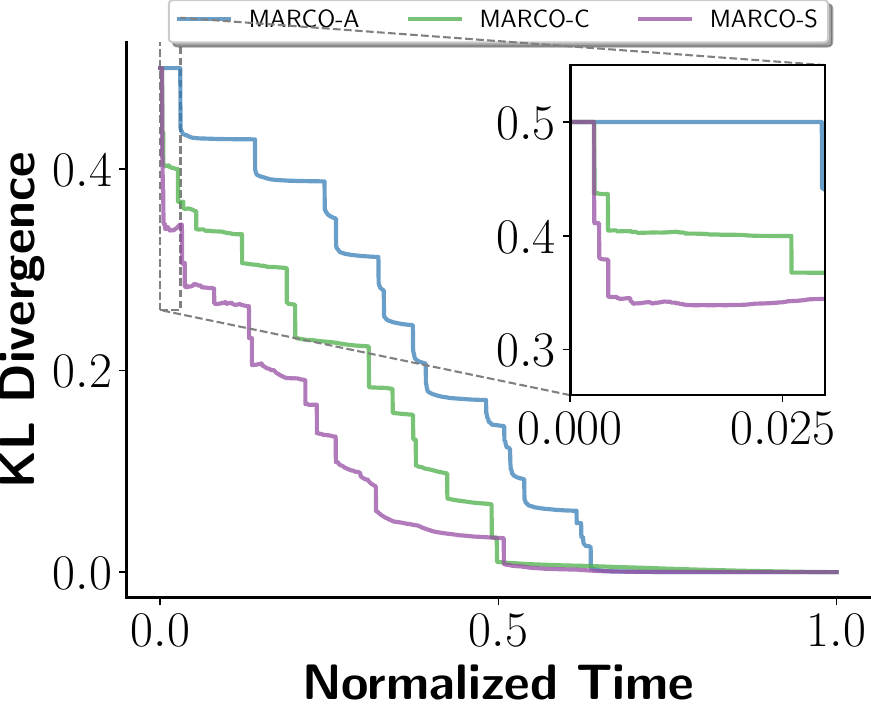}
     \caption{MNIST-1vs7}
    \label{fig:kl-mnist1v7}
   \end{subfigure}
   \begin{subfigure}[b]{0.3\textwidth}
    \centering
    \includegraphics[width=\textwidth]{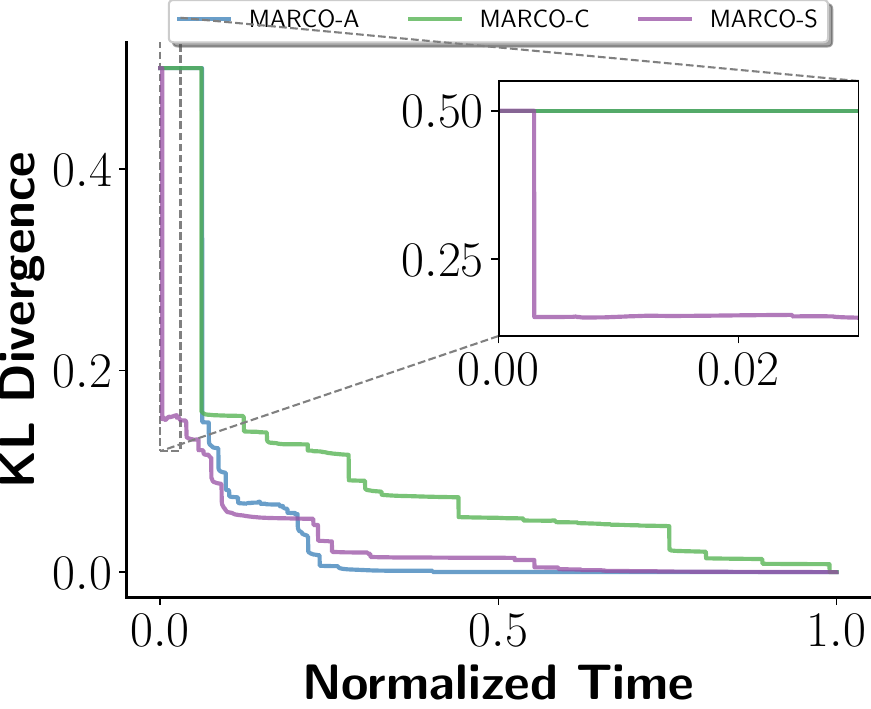}
    \caption{PneumoniaMNIST}
    \label{fig:kl-pneumonia}
   \end{subfigure}%
   \vspace{1em}
   
   \begin{subfigure}[b]{0.3\textwidth}
    \centering
    \includegraphics[width=\textwidth]{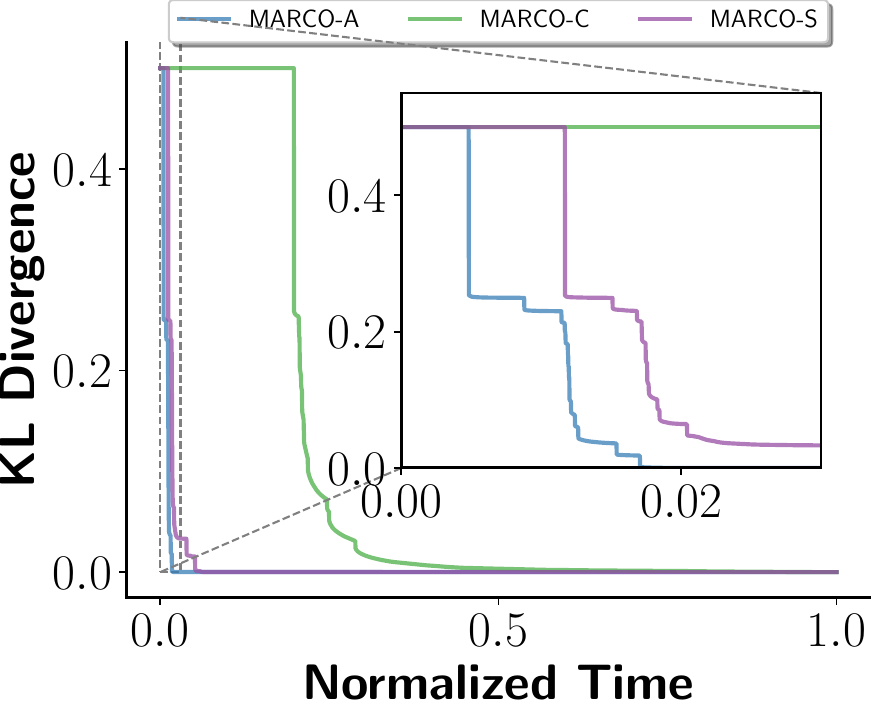}
    \caption{Sarcasm}
    \label{fig:kl-sarcasm}
   \end{subfigure}%
   \hspace{1em}
   \begin{subfigure}[b]{0.3\textwidth}
     \centering
     \includegraphics[width=\textwidth]{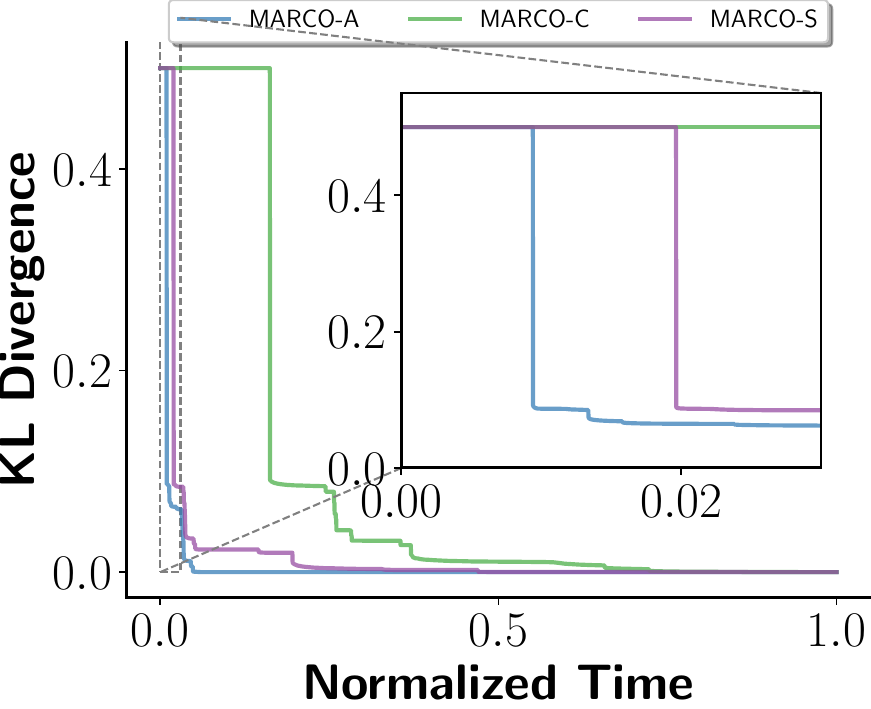}
     \caption{Disaster}
     \label{fig:kl-disaster}
   \end{subfigure}
   \caption{KL-divergence over time in each dataset.}
   \label{fig:kl-dt}
\end{figure*}

\begin{figure*}[!t]
  \centering
   \begin{subfigure}[b]{0.3\textwidth}
    \centering
    \includegraphics[width=\textwidth]{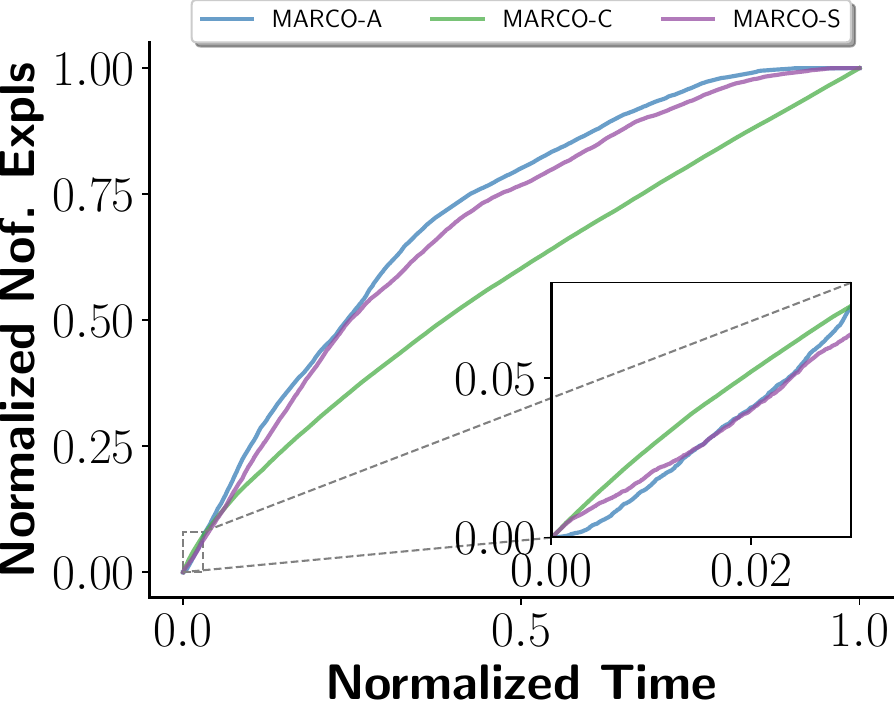}
    \caption{MNIST-1vs3}
     \label{fig:expls-mnist1v3}
   \end{subfigure}%
   \begin{subfigure}[b]{0.3\textwidth}
     \centering
     \includegraphics[width=\textwidth]{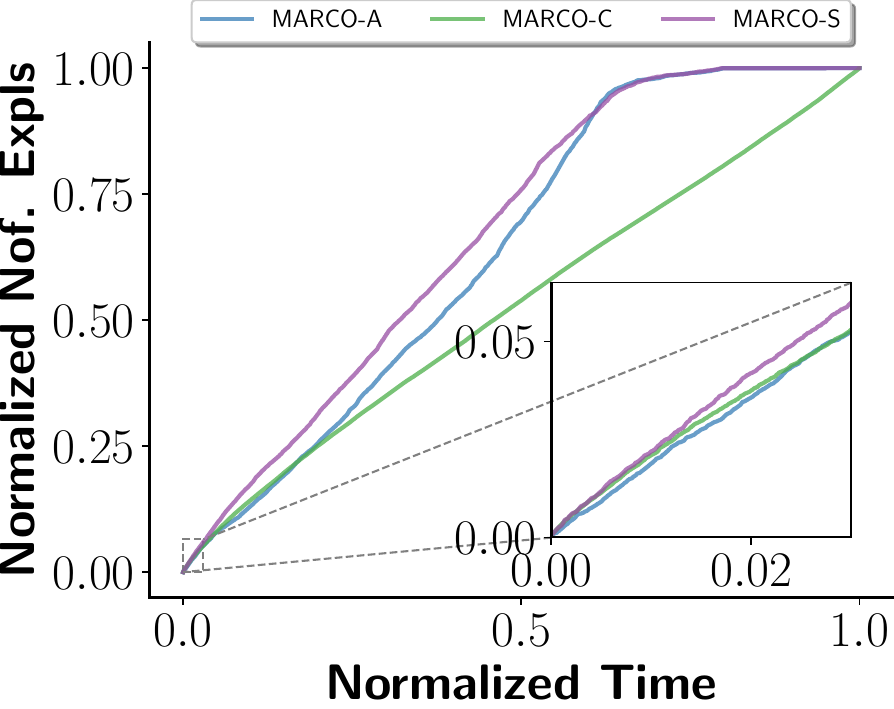}
     \caption{MNIST-1vs7}
     \label{fig:expls-mnist1v7}
   \end{subfigure}
   \begin{subfigure}[b]{0.3\textwidth}
    \centering
    \includegraphics[width=\textwidth]{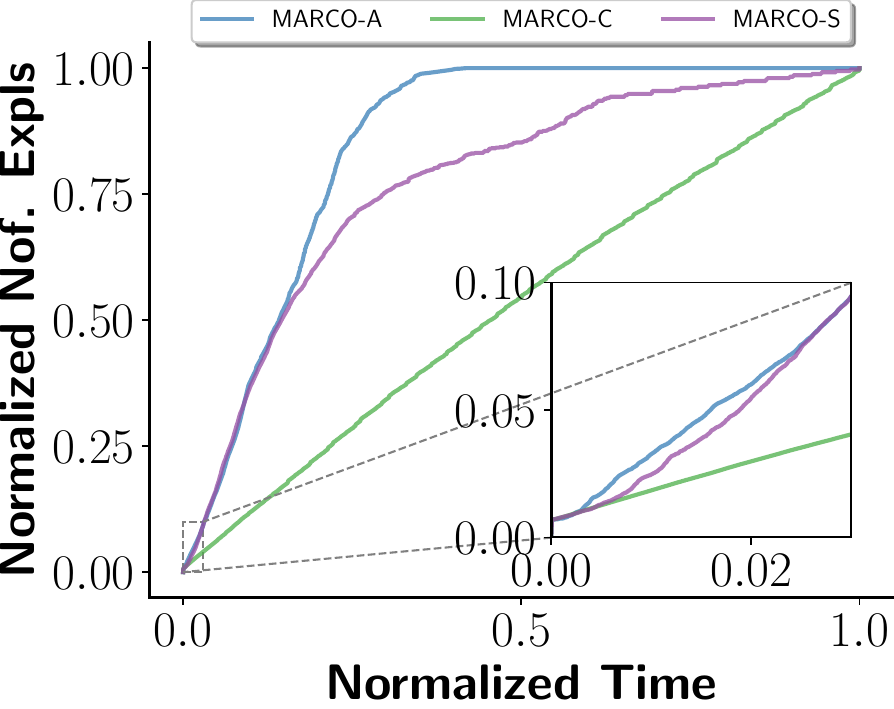}
    \caption{PneumoniaMNIST}
    \label{fig:expls-pneumonia}
   \end{subfigure}
   \vspace{1em}
   
   \begin{subfigure}[b]{0.3\textwidth}
    \centering
    \includegraphics[width=\textwidth]{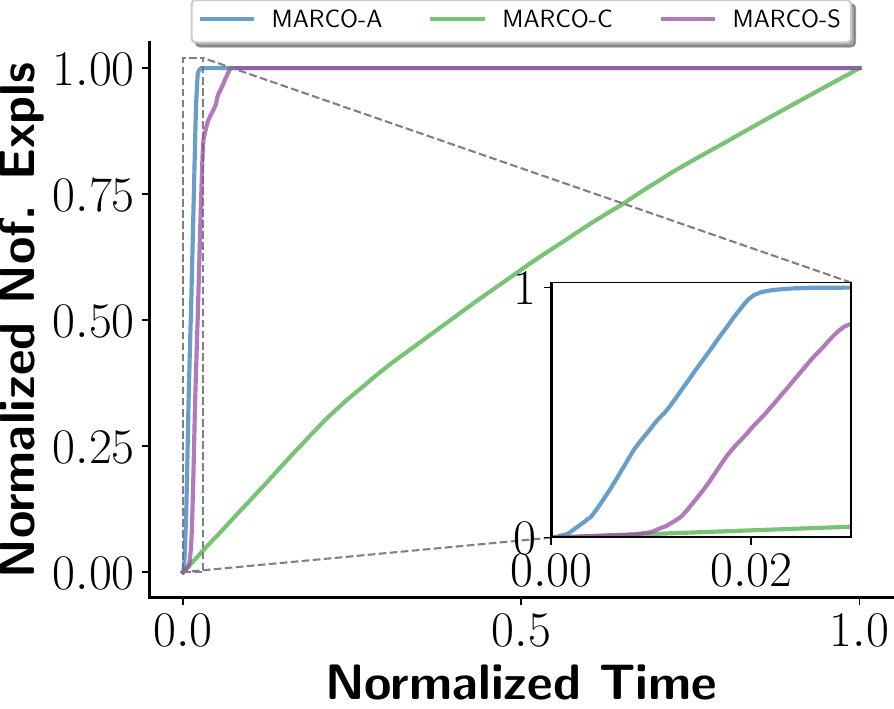}
    \caption{Sarcasm}
    \label{fig:expls-sarcasm}
   \end{subfigure}%
   \hspace{1em}
   \begin{subfigure}[b]{0.3\textwidth}
     \centering
     \includegraphics[width=\textwidth]{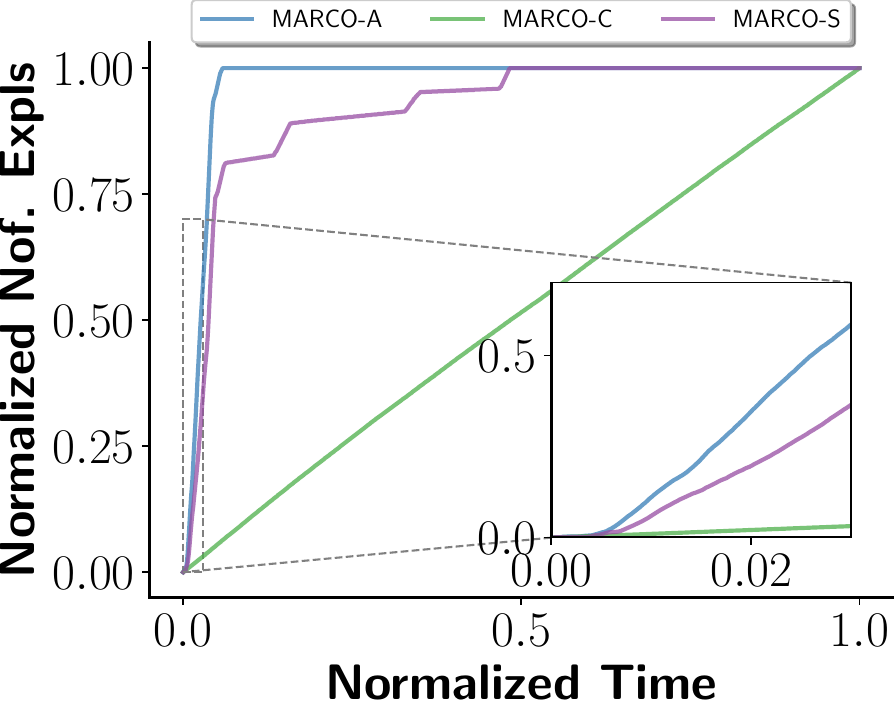}
     \caption{Disaster}
     \label{fig:expls-disaster}
   \end{subfigure}
   \caption{Number of explanations over time in each dataset.}
   \label{fig:expls-dt}
\end{figure*}

%% file: paper.bbl
\begin{thebibliography}{71}
\providecommand{\natexlab}[1]{#1}

\bibitem[{Amgoud and Ben{-}Naim(2022)}]{an-ijcai22}
Amgoud, L.; and Ben{-}Naim, J. 2022.
\newblock Axiomatic Foundations of Explainability.
\newblock In Raedt, L.~D., ed., \emph{IJCAI}, 636--642.

\bibitem[{Arenas et~al.(2021)Arenas, Baez, Barcel{\'{o}}, P{\'{e}}rez, and
  Subercaseaux}]{barcelo-nips21}
Arenas, M.; Baez, D.; Barcel{\'{o}}, P.; P{\'{e}}rez, J.; and Subercaseaux, B.
  2021.
\newblock Foundations of Symbolic Languages for Model Interpretability.
\newblock In \emph{NeurIPS}.

\bibitem[{Arenas et~al.(2022)Arenas, Barcel{\'{o}}, Orth, and
  Subercaseaux}]{barcelo-nips22}
Arenas, M.; Barcel{\'{o}}, P.; Orth, M. A.~R.; and Subercaseaux, B. 2022.
\newblock On Computing Probabilistic Explanations for Decision Trees.
\newblock In \emph{NeurIPS}.

\bibitem[{Bailey and Stuckey(2005)}]{bs-dapl05}
Bailey, J.; and Stuckey, P.~J. 2005.
\newblock Discovery of Minimal Unsatisfiable Subsets of Constraints Using
  Hitting Set Dualization.
\newblock In \emph{PADL}, 174--186.

\bibitem[{Belov, Lynce, and Marques{-}Silva(2012)}]{blms-aicom12}
Belov, A.; Lynce, I.; and Marques{-}Silva, J. 2012.
\newblock Towards efficient {MUS} extraction.
\newblock \emph{{AI} Commun.}, 25(2): 97--116.

\bibitem[{Biere et~al.(2021)Biere, Heule, van Maaren, and
  Walsh}]{sat-handbook21}
Biere, A.; Heule, M.; van Maaren, H.; and Walsh, T., eds. 2021.
\newblock \emph{Handbook of Satisfiability}.

\bibitem[{Biradar et~al.(2023)Biradar, Izza, Lobo, Viswanathan, and
  Zick}]{izza-corr23b}
Biradar, G.; Izza, Y.; Lobo, E.; Viswanathan, V.; and Zick, Y. 2023.
\newblock Axiomatic Aggregations of Abductive Explanations.
\newblock \emph{CoRR}, abs/2310.03131.

\bibitem[{Birnbaum and Lozinskii(2003)}]{birnbaum-jetai03}
Birnbaum, E.; and Lozinskii, E.~L. 2003.
\newblock Consistent subsets of inconsistent systems: structure and behaviour.
\newblock \emph{J. Exp. Theor. Artif. Intell.}, 15(1): 25--46.

\bibitem[{Blanc, Lange, and Tan(2021)}]{tan-nips21}
Blanc, G.; Lange, J.; and Tan, L. 2021.
\newblock Provably efficient, succinct, and precise explanations.
\newblock In \emph{NeurIPS}.

\bibitem[{Boumazouza et~al.(2021)Boumazouza, Alili, Mazure, and
  Tabia}]{mazure-cikm21}
Boumazouza, R.; Alili, F.~C.; Mazure, B.; and Tabia, K. 2021.
\newblock {ASTERYX:} {A} model-{A}gnostic {S}a{T}-bas{E}d app{R}oach for
  s{Y}mbolic and score-based e{X}planations.
\newblock In \emph{CIKM}, 120--129.

\bibitem[{Chen and Guestrin(2016)}]{guestrin-kdd16a}
Chen, T.; and Guestrin, C. 2016.
\newblock {XGBoost}: {A} Scalable Tree Boosting System.
\newblock In \emph{{KDD}}, 785--794.

\bibitem[{Clark and Boswell(1991)}]{clark-ewsl91}
Clark, P.; and Boswell, R. 1991.
\newblock Rule Induction with {CN2:} Some Recent Improvements.
\newblock In \emph{{EWSL}}, 151--163.

\bibitem[{Cooper and Marques{-}Silva(2023)}]{jpms-aij23}
Cooper, M.~C.; and Marques{-}Silva, J. 2023.
\newblock Tractability of explaining classifier decisions.
\newblock \emph{Artif. Intell.}, 316: 103841.

\bibitem[{Darwiche and Marquis(2021)}]{darwiche-jair21}
Darwiche, A.; and Marquis, P. 2021.
\newblock On Quantifying Literals in {Boolean} Logic and Its Applications to
  Explainable {AI}.
\newblock \emph{J. Artif. Intell. Res.}, 72: 285--328.

\bibitem[{Deng(2012)}]{deng2012mnist}
Deng, L. 2012.
\newblock The {MNIST} database of handwritten digit images for machine learning
  research.
\newblock \emph{IEEE Signal Processing Magazine}, 29(6): 141--142.

\bibitem[{Ferreira et~al.(2022)Ferreira, de~Sousa~Ribeiro, Gon{\c{c}}alves, and
  Leite}]{leite-kr22}
Ferreira, J.; de~Sousa~Ribeiro, M.; Gon{\c{c}}alves, R.; and Leite, J. 2022.
\newblock Looking Inside the Black-Box: Logic-based Explanations for Neural
  Networks.
\newblock In \emph{KR}, 432–442.

\bibitem[{Friedman(2001)}]{friedman-tas01}
Friedman, J.~H. 2001.
\newblock Greedy Function Approximation: A Gradient Boosting Machine.
\newblock \emph{The Annals of Statistics}, 29(5): 1189--1232.

\bibitem[{Giunchiglia and Maratea(2006)}]{giunchiglia-ecai06}
Giunchiglia, E.; and Maratea, M. 2006.
\newblock Solving Optimization Problems with {DLL}.
\newblock In \emph{{ECAI}}, 377--381.

\bibitem[{Gorji and Rubin(2022)}]{rubin-aaai22}
Gorji, N.; and Rubin, S. 2022.
\newblock Sufficient Reasons for Classifier Decisions in the Presence of Domain
  Constraints.
\newblock In \emph{{AAAI}}, 5660--5667.

\bibitem[{Howard et~al.(2019)Howard, Devrishi, Culliton, and
  Guo}]{hdcg-kaggle19}
Howard, A.; Devrishi; Culliton, P.; and Guo, Y. 2019.
\newblock Natural Language Processing with Disaster Tweets.

\bibitem[{Huang et~al.(2023)Huang, Cooper, Morgado, Planes, and
  Marques{-}Silva}]{huang-tacas23}
Huang, X.; Cooper, M.~C.; Morgado, A.; Planes, J.; and Marques{-}Silva, J.
  2023.
\newblock Feature Necessity {\&} Relevancy in {ML} Classifier Explanations.
\newblock In \emph{{TACAS} {(1)}}, 167--186.

\bibitem[{Huang et~al.(2022)Huang, Izza, Ignatiev, Cooper, Asher, and
  Marques{-}Silva}]{hiicams-aaai22}
Huang, X.; Izza, Y.; Ignatiev, A.; Cooper, M.~C.; Asher, N.; and
  Marques{-}Silva, J. 2022.
\newblock Tractable Explanations for {d-DNNF} Classifiers.
\newblock In \emph{AAAI}, 5719--5728.

\bibitem[{Huang, Izza, and Marques-Silva(2023)}]{jpms-aaai23}
Huang, X.; Izza, Y.; and Marques-Silva, J. 2023.
\newblock Solving Explainability Queries with Quantification: The Case of
  Feature Relevancy.
\newblock In \emph{AAAI}, 4123--4131.

\bibitem[{Huang and Marques{-}Silva(2023{\natexlab{a}})}]{jpms-ecai23}
Huang, X.; and Marques{-}Silva, J. 2023{\natexlab{a}}.
\newblock From Decision Trees to Explained Decision Sets.
\newblock In \emph{{ECAI}}, 1100--1108.

\bibitem[{Huang and Marques{-}Silva(2023{\natexlab{b}})}]{huang-corr23}
Huang, X.; and Marques{-}Silva, J. 2023{\natexlab{b}}.
\newblock The Inadequacy of {Shapley} Values for Explainability.
\newblock \emph{CoRR}, abs/2302.08160.

\bibitem[{Hubara et~al.(2016)Hubara, Courbariaux, Soudry, El{-}Yaniv, and
  Bengio}]{hcseyb-neurips16}
Hubara, I.; Courbariaux, M.; Soudry, D.; El{-}Yaniv, R.; and Bengio, Y. 2016.
\newblock Binarized Neural Networks.
\newblock In \emph{NIPS}, 4107--4115.

\bibitem[{Hyafil and Rivest(1976)}]{rivest-ipl76}
Hyafil, L.; and Rivest, R.~L. 1976.
\newblock Constructing Optimal Binary Decision Trees is {NP}-Complete.
\newblock \emph{Inf. Process. Lett.}, 5(1): 15--17.

\bibitem[{Ignatiev et~al.(2022)Ignatiev, Izza, Stuckey, and
  Marques{-}Silva}]{iisms-aaai22}
Ignatiev, A.; Izza, Y.; Stuckey, P.~J.; and Marques{-}Silva, J. 2022.
\newblock Using {MaxSAT} for Efficient Explanations of Tree Ensembles.
\newblock In \emph{AAAI}, 3776--3785.

\bibitem[{Ignatiev and Marques-Silva(2021)}]{ims-sat21}
Ignatiev, A.; and Marques-Silva, J. 2021.
\newblock {SAT}-Based Rigorous Explanations for Decision Lists.
\newblock In \emph{SAT}, 251--269.

\bibitem[{Ignatiev et~al.(2020)Ignatiev, Narodytska, Asher, and
  Marques{-}Silva}]{inams-aiia20}
Ignatiev, A.; Narodytska, N.; Asher, N.; and Marques{-}Silva, J. 2020.
\newblock From Contrastive to Abductive Explanations and Back Again.
\newblock In \emph{AI*IA}, 335--355.

\bibitem[{Ignatiev, Narodytska, and Marques{-}Silva(2019)}]{inms-aaai19}
Ignatiev, A.; Narodytska, N.; and Marques{-}Silva, J. 2019.
\newblock Abduction-Based Explanations for Machine Learning Models.
\newblock In \emph{AAAI}, 1511--1519.

\bibitem[{Izza et~al.(2023{\natexlab{a}})Izza, Huang, Ignatiev, Narodytska,
  Cooper, and Marques-Silva}]{ihincms-ajar23}
Izza, Y.; Huang, X.; Ignatiev, A.; Narodytska, N.; Cooper, M.~C.; and
  Marques-Silva, J. 2023{\natexlab{a}}.
\newblock On Computing Probabilistic Abductive Explanations.
\newblock \emph{International Journal of Approximate Reasoning}, 159.

\bibitem[{Izza, Ignatiev, and Marques{-}Silva(2022)}]{iims-jair22}
Izza, Y.; Ignatiev, A.; and Marques{-}Silva, J. 2022.
\newblock On Tackling Explanation Redundancy in Decision Trees.
\newblock \emph{J. Artif. Intell. Res.}, 75: 261--321.

\bibitem[{Izza et~al.(2023{\natexlab{b}})Izza, Ignatiev, Stuckey, and
  Marques{-}Silva}]{izza-corr23}
Izza, Y.; Ignatiev, A.; Stuckey, P.~J.; and Marques{-}Silva, J.
  2023{\natexlab{b}}.
\newblock Delivering Inflated Explanations.
\newblock \emph{CoRR}, abs/2306.15272.

\bibitem[{Izza and Marques{-}Silva(2021)}]{ims-ijcai21}
Izza, Y.; and Marques{-}Silva, J. 2021.
\newblock On Explaining Random Forests with {SAT}.
\newblock In \emph{{IJCAI}}.

\bibitem[{Kendall(1938)}]{kendall1938}
Kendall, M.~G. 1938.
\newblock A new measure of rank correlation.
\newblock \emph{Biometrika}, 30(1/2): 81--93.

\bibitem[{Kullback and Leibler(1951)}]{kl-ams51}
Kullback, S.; and Leibler, R.~A. 1951.
\newblock On information and sufficiency.
\newblock \emph{The annals of mathematical statistics}, 22(1): 79--86.

\bibitem[{Lakkaraju, Bach, and Leskovec(2016)}]{lakkaraju-kdd16}
Lakkaraju, H.; Bach, S.~H.; and Leskovec, J. 2016.
\newblock Interpretable Decision Sets: {A} Joint Framework for Description and
  Prediction.
\newblock In \emph{{KDD}}, 1675--1684. {ACM}.

\bibitem[{Liffiton and Malik(2013)}]{liffiton-cpaior13}
Liffiton, M.; and Malik, A. 2013.
\newblock Enumerating Infeasibility: Finding Multiple MUSes Quickly.
\newblock In \emph{{CPAIOR}}, 160--175.

\bibitem[{Liffiton et~al.(2016)Liffiton, Previti, Malik, and
  Marques{-}Silva}]{lpmms-cj16}
Liffiton, M.~H.; Previti, A.; Malik, A.; and Marques{-}Silva, J. 2016.
\newblock Fast, flexible {MUS} enumeration.
\newblock \emph{Constraints An Int. J.}, 21(2): 223--250.

\bibitem[{Liffiton and Sakallah(2008)}]{liffiton-jar08}
Liffiton, M.~H.; and Sakallah, K.~A. 2008.
\newblock Algorithms for Computing Minimal Unsatisfiable Subsets of
  Constraints.
\newblock \emph{J. Autom. Reasoning}, 40(1): 1--33.

\bibitem[{Lundberg and Lee(2017)}]{lundberg-nips17}
Lundberg, S.~M.; and Lee, S. 2017.
\newblock A Unified Approach to Interpreting Model Predictions.
\newblock In \emph{NeurIPS}, 4765--4774.

\bibitem[{Malfa et~al.(2021)Malfa, Michelmore, Zbrzezny, Paoletti, and
  Kwiatkowska}]{kwiatkowska-ijcai21}
Malfa, E.~L.; Michelmore, R.; Zbrzezny, A.~M.; Paoletti, N.; and Kwiatkowska,
  M. 2021.
\newblock On Guaranteed Optimal Robust Explanations for {NLP} Models.
\newblock In \emph{IJCAI}, 2658--2665.

\bibitem[{Marques{-}Silva(2022{\natexlab{a}})}]{jpms-rw22}
Marques{-}Silva, J. 2022{\natexlab{a}}.
\newblock Logic-Based Explainability in Machine Learning.
\newblock In \emph{Reasoning Web}, 24--104.

\bibitem[{Marques{-}Silva(2022{\natexlab{b}})}]{jpms-corr22}
Marques{-}Silva, J. 2022{\natexlab{b}}.
\newblock Logic-Based Explainability in Machine Learning.
\newblock \emph{CoRR}, abs/2211.00541.

\bibitem[{Marques{-}Silva(2023)}]{jpms-corr23}
Marques{-}Silva, J. 2023.
\newblock Disproving {XAI} Myths with Formal Methods - Initial Results.
\newblock \emph{CoRR}, abs/2306.01744.

\bibitem[{Marques{-}Silva et~al.(2020)Marques{-}Silva, Gerspacher, Cooper,
  Ignatiev, and Narodytska}]{msgcin-nips20}
Marques{-}Silva, J.; Gerspacher, T.; Cooper, M.~C.; Ignatiev, A.; and
  Narodytska, N. 2020.
\newblock Explaining Naive {Bayes} and Other Linear Classifiers with Polynomial
  Time and Delay.
\newblock In \emph{NeurIPS}.

\bibitem[{Marques{-}Silva et~al.(2021)Marques{-}Silva, Gerspacher, Cooper,
  Ignatiev, and Narodytska}]{msgcin-icml21}
Marques{-}Silva, J.; Gerspacher, T.; Cooper, M.~C.; Ignatiev, A.; and
  Narodytska, N. 2021.
\newblock Explanations for Monotonic Classifiers.
\newblock In \emph{ICML}, 7469--7479.

\bibitem[{Marques{-}Silva et~al.(2013)Marques{-}Silva, Heras, Janota, Previti,
  and Belov}]{mshjpb-ijcai13}
Marques{-}Silva, J.; Heras, F.; Janota, M.; Previti, A.; and Belov, A. 2013.
\newblock On Computing Minimal Correction Subsets.
\newblock In \emph{{IJCAI}}, 615--622.

\bibitem[{Marques{-}Silva and Huang(2023)}]{huang-corr23b}
Marques{-}Silva, J.; and Huang, X. 2023.
\newblock Explainability is {NOT} a Game.
\newblock \emph{CoRR}, abs/2307.07514.

\bibitem[{Marques{-}Silva and Ignatiev(2022)}]{msi-aaai22}
Marques{-}Silva, J.; and Ignatiev, A. 2022.
\newblock Delivering Trustworthy {AI} through Formal {XAI}.
\newblock In \emph{{AAAI}}, 12342--12350.

\bibitem[{Marques-Silva and Ignatiev(2023)}]{msi-fai23}
Marques-Silva, J.; and Ignatiev, A. 2023.
\newblock No Silver Bullet: Interpretable ML Models Must Be Explained.
\newblock \emph{Frontiers in Artificial Intelligence}, 6: 1--15.

\bibitem[{Miller(2019)}]{miller-aij19}
Miller, T. 2019.
\newblock Explanation in artificial intelligence: Insights from the social
  sciences.
\newblock \emph{Artif. Intell.}, 267: 1--38.

\bibitem[{Misra and Arora(2023)}]{misra2023Sarcasm}
Misra, R.; and Arora, P. 2023.
\newblock Sarcasm Detection using News Headlines Dataset.
\newblock \emph{AI Open}, 4: 13--18.

\bibitem[{Misra and Grover(2021)}]{misra2021sculpting}
Misra, R.; and Grover, J. 2021.
\newblock \emph{Sculpting Data for ML: The first act of Machine Learning}.
\newblock ISBN 9798585463570.

\bibitem[{Nair and Hinton(2010)}]{hinton-icml10}
Nair, V.; and Hinton, G. 2010.
\newblock Rectified Linear Units Improve Restricted Boltzmann Machines.
\newblock In \emph{{ICML}}, 807--814.

\bibitem[{Paszke et~al.(2019)Paszke, Gross, Massa, Lerer, Bradbury, Chanan,
  Killeen, Lin, Gimelshein, Antiga, Desmaison, K{\"{o}}pf, Yang, DeVito,
  Raison, Tejani, Chilamkurthy, Steiner, Fang, Bai, and
  Chintala}]{pytorch-neurips19}
Paszke, A.; Gross, S.; Massa, F.; Lerer, A.; Bradbury, J.; Chanan, G.; Killeen,
  T.; Lin, Z.; Gimelshein, N.; Antiga, L.; Desmaison, A.; K{\"{o}}pf, A.; Yang,
  E.~Z.; DeVito, Z.; Raison, M.; Tejani, A.; Chilamkurthy, S.; Steiner, B.;
  Fang, L.; Bai, J.; and Chintala, S. 2019.
\newblock {PyTorch}: An Imperative Style, High-Performance Deep Learning
  Library.
\newblock In \emph{NeurIPS}, 8024--8035.

\bibitem[{Previti and Marques{-}Silva(2013)}]{pms-aaai13}
Previti, A.; and Marques{-}Silva, J. 2013.
\newblock Partial {MUS} Enumeration.
\newblock In \emph{{AAAI}}. {AAAI} Press.

\bibitem[{Provan and Ball(1983)}]{provan-sicomp83}
Provan, J.~S.; and Ball, M.~O. 1983.
\newblock The complexity of counting cuts and of computing the probability that
  a graph is connected.
\newblock \emph{SIAM J. Comput.}, 12(4): 777--788.

\bibitem[{Reiter(1987)}]{reiter-aij87}
Reiter, R. 1987.
\newblock A Theory of Diagnosis from First Principles.
\newblock \emph{Artif. Intell.}, 32(1): 57--95.

\bibitem[{Ribeiro, Singh, and Guestrin(2016)}]{guestrin-kdd16}
Ribeiro, M.~T.; Singh, S.; and Guestrin, C. 2016.
\newblock "{Why} Should {I} Trust You?": Explaining the Predictions of Any
  Classifier.
\newblock In \emph{KDD}, 1135--1144.

\bibitem[{Rivest(1987)}]{rivest-ml87}
Rivest, R.~L. 1987.
\newblock Learning Decision Lists.
\newblock \emph{Mach. Learn.}, 2(3): 229--246.

\bibitem[{Shih, Choi, and Darwiche(2018)}]{darwiche-ijcai18}
Shih, A.; Choi, A.; and Darwiche, A. 2018.
\newblock A Symbolic Approach to Explaining {Bayesian} Network Classifiers.
\newblock In \emph{IJCAI}, 5103--5111.

\bibitem[{Vadhan(2001)}]{vadhan-sicomp01}
Vadhan, S. 2001.
\newblock The complexity of counting in sparse, regular, and planar graphs.
\newblock \emph{SIAM J. Comput.}, 31(2): 398--427.

\bibitem[{Valiant(1979)}]{valiant-tcs79}
Valiant, L.~G. 1979.
\newblock The complexity of computing the permanent.
\newblock \emph{Theoret. Comput. Sci.}, 8(2): 189--201.

\bibitem[{W{\"{a}}ldchen et~al.(2021)W{\"{a}}ldchen, MacDonald, Hauch, and
  Kutyniok}]{kutyniok-jair21}
W{\"{a}}ldchen, S.; MacDonald, J.; Hauch, S.; and Kutyniok, G. 2021.
\newblock The Computational Complexity of Understanding Binary Classifier
  Decisions.
\newblock \emph{J. Artif. Intell. Res.}, 70: 351--387.

\bibitem[{Webber, Moffat, and Zobel(2010)}]{wmz-tois10}
Webber, W.; Moffat, A.; and Zobel, J. 2010.
\newblock A similarity measure for indefinite rankings.
\newblock \emph{ACM Transactions on Information Systems (TOIS)}, 28(4): 1--38.

\bibitem[{Yang et~al.(2023)Yang, Shi, Wei, Liu, Zhao, Ke, Pfister, and
  Ni}]{medmnistv2}
Yang, J.; Shi, R.; Wei, D.; Liu, Z.; Zhao, L.; Ke, B.; Pfister, H.; and Ni, B.
  2023.
\newblock {MedMNIST} v2-A large-scale lightweight benchmark for {2D} and {3D}
  biomedical image classification.
\newblock \emph{Scientific Data}, 10(1): 41.

\bibitem[{Yu, Ignatiev, and Stuckey(2023{\natexlab{a}})}]{yis-cp23}
Yu, J.; Ignatiev, A.; and Stuckey, P.~J. 2023{\natexlab{a}}.
\newblock From Formal Boosted Tree Explanations to Interpretable Rule Sets.
\newblock In \emph{{CP}}, volume 280 of \emph{LIPIcs}, 38:1--38:21.

\bibitem[{Yu, Ignatiev, and Stuckey(2023{\natexlab{b}})}]{yis-corr23}
Yu, J.; Ignatiev, A.; and Stuckey, P.~J. 2023{\natexlab{b}}.
\newblock On Formal Feature Attribution and Its Approximation.
\newblock \emph{CoRR}, abs/2307.03380.

\bibitem[{Yu et~al.(2023)Yu, Ignatiev, Stuckey, Narodytska, and
  Marques-Silva}]{yisnms-aaai23}
Yu, J.; Ignatiev, A.; Stuckey, P.~J.; Narodytska, N.; and Marques-Silva, J.
  2023.
\newblock Eliminating The Impossible, Whatever Remains Must Be True.
\newblock In \emph{AAAI}, 4123--4131.

\end{thebibliography}
